\documentclass[11pt]{article}
\usepackage{amsmath,amssymb}
\usepackage{amsthm}
\usepackage{mathtools}
\usepackage[noend]{algorithmic}
\usepackage[ruled,vlined]{algorithm2e}
\usepackage{url}
\usepackage{fullpage}
\usepackage{makeidx}
\usepackage{enumerate}
\usepackage[top=1in, bottom=1.25in, left=1in, right=1in]{geometry}
\usepackage{graphicx,float,psfrag,epsfig,caption}
\usepackage[usenames,dvipsnames,svgnames,table]{xcolor}
\definecolor{darkgreen}{rgb}{0.0,0,0.9}
\usepackage[pagebackref,letterpaper=true,colorlinks=true,pdfpagemode=none,citecolor=OliveGreen,linkcolor=BrickRed,urlcolor=BrickRed,pdfstartview=FitH]{hyperref}
\usepackage{epstopdf}
\usepackage{color}
\usepackage{xr}
\usepackage{subfig}
\usepackage{caption}
\usepackage{graphicx}
\usepackage[utf8]{inputenc}

\usepackage{scalerel,stackengine}

\stackMath
\newcommand\reallywidehat[1]{%
\savestack{\tmpbox}{\stretchto{%
  \scaleto{%
    \scalerel*[\widthof{\ensuremath{#1}}]{\kern.1pt\mathchar"0362\kern.1pt}%
    {\rule{0ex}{\textheight}}
  }{\textheight}%
}{2.4ex}}%
\stackon[-6.9pt]{#1}{\tmpbox}%
}
\usepackage[mathscr]{euscript}

\DeclareSymbolFont{rsfs}{U}{rsfs}{m}{n}
\DeclareSymbolFontAlphabet{\mathscrsfs}{rsfs}

\numberwithin{equation}{section}

\newtheoremstyle{myexample} 
    {\topsep}                    
    {\topsep}                    
    {\rm }                   
    {}                           
    {\bf }                   
    {.}                          
    {.5em}                       
    {}  

\newtheoremstyle{myremark} 
    {\topsep}                    
    {\topsep}                    
    {\rm}                        
    {}                           
    {\bf}                        
    {.}                          
    {.5em}                       
    {}  

\newtheorem{claim}{Claim}[section]
\newtheorem{lemma}[claim]{Lemma}

\newtheorem{proposition}[claim]{Proposition}

\newtheorem{definition}[claim]{Definition}

\theoremstyle{myremark}
\newtheorem{remark}{Remark}[section]

\theoremstyle{myremark}

\theoremstyle{myexample}

\definecolor{darkgreen}{rgb}{0.0, 0.5, 0.0}

\newcommand{\bea}{\begin{eqnarray}}
\newcommand{\eea}{\end{eqnarray}}
\newcommand{\<}{\langle}
\renewcommand{\>}{\rangle}
\newcommand{\E}{{\mathbb E}}

\definecolor{darkblue}{rgb}{0, 0, 0.5}

\newcommand{\rev}[1]{\textcolor{black}{#1}}

\def\sTV{\mbox{\tiny \rm TV}}

\def\Unif{{\sf Unif}}

\def\eps{{\varepsilon}}

\def\id{{\boldsymbol{I}}}

\def\cuP{\mathscrsfs{P}}
\def\cuF{\mathscrsfs{F}}
\def\cuO{\mathscrsfs{O}}

\def\Conv{{\sf Conv}}
\def\Lin{{\sf Lin}}
\def\ReLU{{\sf ReLU}}

\def\bm{{\boldsymbol{m}}}
\def\btheta{{\boldsymbol{\theta}}}

\def\bZ{{\boldsymbol{Z}}}
\def\hbSigma{\hat{\boldsymbol{\Sigma}}}
\def\bSigma{{\boldsymbol{\Sigma}}}

\newcommand{\indep}{\perp \!\!\! \perp}

\def\hsP{\hat{\sf P}}
\def\sP{{\sf P}}

\def\bQ{{\boldsymbol{Q}}}
\def\bM{{\boldsymbol{M}}}

\def\obY{\overline{\boldsymbol{Y}}}
\def\tbY{\tilde{\boldsymbol{Y}}}

\def\hq{{\widehat{q}}}

\def\bg{{\boldsymbol{g}}}

\def\bzero{{\mathbf 0}}

\def\cF{{\mathcal F}}

\def\cX{{\mathcal X}}

\def\out{\mbox{\tiny\rm out}}
\def\slin{\mbox{\tiny\rm lin}}

\def\naturals{{\mathbb N}}
\def\reals{{\mathbb R}}

\def\integers{{\mathbb Z}}

\def\normal{{\sf N}}
\def\Exp{{\sf Exp}}

\def\sT{{\sf T}}

\def\Ent{{\rm Ent}}
\def\MMSE{{\rm MMSE}}

\def\of{\overline{f}}

\def\obx{\overline{\boldsymbol{x}}}
\def\bv{{\boldsymbol{v}}}
\def\bz{{\boldsymbol{z}}}
\def\bx{{\boldsymbol{x}}}
\def\ba{{\boldsymbol{a}}}

\def\bA{\boldsymbol{A}}
\def\bB{\boldsymbol{B}}

\def\bL{\boldsymbol{L}}
\def\bT{\boldsymbol{T}}
\def\bX{\boldsymbol{X}}

\def\hbY{\boldsymbol{\hat{Y}}}

\def\fG{{\mathfrak G}}
\def\hc{\hat{c}}

\def\de{{\rm d}}

\def\bX{\boldsymbol{X}}
\def\bY{\boldsymbol{Y}}
\def\bW{\boldsymbol{W}}
\def\prob{{\mathbb P}}
\def\hprob{\hat{\mathbb P}}
\def\E{{\mathbb E}}

\def\<{\langle}
\def\>{\rangle}

\def\ed{\stackrel{{\rm d}}{=}}
\def\Poisson{{\rm Poisson}}
\def\PPP{{\rm PPP}}
\def\cH{{\cal H}}

\def\cY{{\cal Y}}

\def\cL{{\cal L}}
\def\hr{\hat{r}}

\def\S{{\mathbb S}}

\def\by{{\boldsymbol{y}}}

\def\hp{\hat{p}}

\def\bD{{\boldsymbol{D}}}

\def\bOmega{{\boldsymbol{\Omega}}}

\def\hP{\hat{P}}

\def\bphi{{\boldsymbol{\phi}}}
\def\bpsi{{\boldsymbol{\psi}}}

\def\b0{{\boldsymbol{0}}}
\def\bq{{\boldsymbol{q}}}

\def\bU{{\boldsymbol{U}}}
\def\bT{{\boldsymbol{T}}}

\def\sav{\mbox{\tiny\rm av}}
\def\sgen{\mbox{\tiny\rm gen}}

\def\bfone{{\boldsymbol 1}}

\def\bF{{\boldsymbol F}}
\def\bG{{\boldsymbol G}}

\def\bm{{\boldsymbol m}}

\def\og{\overline{g}}

\def\ft{{\mathfrak t}}
\def\fs{{\mathfrak s}}

\def\obB{\overline{\boldsymbol B}}
\def\obF{\overline{\boldsymbol F}}

\def\bS{{\boldsymbol S}}

\def\hbm{\hat{\boldsymbol m}}

\def\bs{{\boldsymbol s}}

\def\rP{{\rm P}}

\def\cA{{\mathcal A}}

\def\cB{{\mathcal B}}

\def\argmin{{\rm arg\,min}}

\def\star{*}

\def\bfzero{\boldsymbol{0}}

\def\rP{{P}}
\def\hrP{\hat{P}}
\def\hmu{{\hat{\mu}}}
\def\hm{{\hat{m}}}
\def\hR{{\hat{R}}}

\title{Sampling, Diffusions, and Stochastic Localization}

\author{Andrea Montanari\thanks{Department of Statistics and Department of Mathematics,
Stanford University. This work was initially carried out at 
Granica Computing Inc. $|$ {\sf granica.ai}}}

\begin{document}

\maketitle

\begin{abstract}
Diffusions are a successful technique to sample from high-dimensional distributions.
The target distribution 
can be either explicitly given or learnt from a collection of samples.
They implement a diffusion process whose endpoint is a sample from the target distribution.
The drift of the diffusion process is typically represented as a neural network.
Stochastic localization is a successful technique to prove mixing of Markov Chains
and other functional inequalities in high dimension. An algorithmic
version of stochastic localization was recently proposed in order to 
sample from certain statistical mechanics models. 

This expository article has three objectives: 
$(i)$~Generalize the algorithmic construction to other stochastic localization processes.
This construction is both simple and broadly applicable;
$(ii)$~Clarify the connection between diffusions and stochastic localization.
This allows to derive several known sampling schemes in a unified fashion;
$(iii)$~Describe the insights that follow from this unified viewpoint.
\end{abstract}

\section{Introduction}

\subsection{Sampling}
\label{sec:FirstSampling}

We would like to generate a sample from a probability distribution $\mu$ 
in $n$ dimensions (here and below, $\cuP(\Omega)$ denotes the set of probability distributions on $\Omega$):
\begin{align}
\bx\sim \mu \, \;\;\;\; \mbox{for}\;\;\mu\in \cuP(\reals^n)\, ,
\end{align}
We have in mind two types of settings:
\rev{
\begin{description}
\item[1. Known $\mu$.] The probability distribution $\mu$ is `explicitly given.'
As an example, we can imagine that $\mu=\mu_{\btheta}$
takes the following exponential form
\begin{align}\label{eq:ExponentialForm}
\mu_{\btheta}(\de\bx) = \frac{1}{Z(\btheta)} \,\exp\big\{-H(\bx;\btheta)\big\}\, \de\bx\, ,
\end{align}
where the function $H(\,\cdot\,;\btheta): \reals^d\to\reals$ is explicit,
 can be efficiently evaluated, and depends on known parameters $\btheta$.
This is typically the case in statistical physics and in Bayesian statistics
(in this case $\mu_{\btheta}(\de\bx)$ should be interpreted as the posterior over unknown object $\bx$).
\item[2. Unknown $\mu$.] 
We do not know $\mu$ but have access to a collection of i.i.d. samples $\bx^{(1)},\dots,\bx^{(N)}\sim_{iid}\mu$,
and want to generate fresh samples with (approximately) the same distribution. 
This might arise in two cases: $(2a)$~We do not make any explicit assumptions about $\mu$.
This is the case in `generative modeling' in machine learning.
$(2b)$~We have a statistical model, e.g. a parametric model of the form \eqref{eq:ExponentialForm},
 but the parameters $\btheta$ are unknown.
\end{description}}

\rev{The methods discussed in this articles are relevant for both settings.
Our main focus will be on the structural properties of the stochastic process
that generates samples $\bx\sim \mu$. The two settings described above
differ in the implementation of a certain `denoising' oracle used by this process: 
we will discuss the latter in Section \ref{sec:Oracle}. (Here we use the term `denoising' 
in a broader sense than its normal usage: see below.)} 

Monte Carlo Markov Chain (MCMC) attempts to solve the problem in the first case
(explicitly known $\mu$)
by constructing a Markov Chain whose stationary distribution coincides with $\mu$,
and sampling from the Markov Chain starting from a fixed initialization. 
If the probability measure $\mu$ is supported on $\Omega\subseteq \reals^n$, 
the Markov chain can be thought as a random walk on the set $\Omega$.
\emph{In the case of unknown $\mu$, the MCMC approach requires first to estimate $\mu$
from the samples  $\bx^{(1)},\dots,\bx^{(N)}$, and then to construct a chain that samples from
the estimated distribution $\hmu$.}

Here, we will consider a class of sampling algorithms that
generate a stochastic process $\bm_t\in\reals^n$, $t\in [0,T]$ (possibly $T=\infty$) such that
$\bm_0$ is deterministic and, as $t\to T$
\begin{align}
\bm_t  \longrightarrow\;\; \bx = \bm_{T}\sim \mu\, .
\end{align}
In contrast with the MCMC setting, this process will
be ---in general--- non-reversible and time-inhomogeneous.
Further, the distribution of $\bm_t$ for $t<T$ is different from $\mu$ and
---in general--- $\bm_t$ does not takes values in the support of $\mu$.

In the rest of this introduction, we will present two approaches towards constructing such 
a process $\bm_t$, respectively based on time-reversal of diffusion processes and on
stochastic localization. We overview the literature
in Section \ref{sec:Literature}, generalize the stochastic localization 
approach in Section \ref{sec:GeneralSL}, and describe several specific instantiations in
Section \ref{sec:Examples}. 
While earlier sections are mainly expository, 
in Sections \ref{sec:Multimodal}, \ref{sec:ToyNum}, \ref{sec:ToyAnalytical}
we illustrate how the present point of view can be exploited
to address problems arising with standard denoising diffusions.
Namely, we show how the choice of the stochastic localization process can 
simplify the learning task.
The appendices contain omitted technical details. In particular, in
Appendix \ref{sec:Loss} we spell out the form taken by a natural 
loss functions (Kullback-Leibler divergence) in various examples.

\subsection{Diffusions}
\label{sec:FirstDiffusions}

Sampling algorithms based on diffusions were introduced in
 \cite{sohl2015deep,song2019generative,ho2020denoising,song2021score} 
and were originally motivated by the idea of time-reversal.
Fixing $S\in (0,\infty]$, the construction starts with a 
It\^o diffusion process $(\bZ_s)_{s\in [0,S]}$ initialized at $\bZ_0= \bx\sim \mu$:
\begin{align}
\de  \bZ_s = \bF(s,\bZ_s) \, \de s+ \sqrt{g(s)}\, \de \bB_s\, ,
\;\;\;\; \bZ_0 = \bx\sim \mu(\,\cdot\,)\, .\label{eq:FirstDiffusion}
\end{align}
Here $(\bB_s)_{s\ge 0 }$ is a standard $n$-dimensional Brownian motion
$\bF:[0,S]\times \reals^n\to\reals^n$ is a drift term, and $g:[0,S]\to\reals_{\ge 0}$
is a diffusion coefficient\footnote{More general diffusion terms can be of interest as well, see 
Section \ref{sec:GeneralSL}. We focus for the moment on the simplest formulation.}.

We will denote by $\mu^{\bZ}_{s}$ the marginal distribution of $\bZ_s$ under the above process, so that, by construction 
$\mu_0^{\bZ}=\mu$. 
The drift and diffusion coefficients in Eq.~\eqref{eq:FirstDiffusion}
can be constructed so that the final distribution $\mu^{\bZ}_S=: \nu$ is easy to sample 
from.

Next, a sampling process is obtained by time-reversing 
the process \eqref{eq:FirstDiffusion}. Namely, we let $\ft:[0,S]\to [0,T]$
be a continuously differentiable time change with first derivative 
$\ft'(s)<0$ for all $s\in [0,S]$, 
$\ft(0) = T$, $\ft(S)=0$, and, let $\fs:[0,T]\to [0,S]$ denote its inverse.
We then
define the process $(\obY_t)_{t\in [0,T]}$ via
\begin{align}
\de  \obY_t = \obF(t,\bY_t) \, \de t+ \sqrt{\og(t)}\, \de \obB_t\, ,
\;\;\;\; \obY_0 \sim \nu(\,\cdot\,)\, ,\label{eq:ReversedDiffusion}
\end{align}
where $(\obB_t)_{t\ge 0}$ is a standard Brownian motion, and the drift and 
diffusion coefficients are given by:
\begin{align}
\obF(t,\by) &= \big[-\bF(\fs(t),\by) +g(\fs(t))\nabla_{\bz}\log\mu^{\bZ}_{\fs(t)}(\by)
\big]\, |\fs'(t)|\, ,\label{eq:ReverseDrift}\\
\og(t) &= g(\fs(t))\, |\fs'(t)|\, .
\end{align}
It is a well known result \cite{haussmann1986time}
that $(\obY_t)_{t\in [0,T]}$ so defined is distributed as $(\bZ_{\fs(t)})_{t\in [0,T]}$,
and in particular $\obY_T \sim \mu(\, \cdot\, )$. In particular this implies
that for each $t\in [0,T]$
\begin{align}
\mu^{\obY}_t = \mu^{\bZ}_{\fs(t)}\, .
\end{align}
Hence the stochastic differential equation (SDE) \eqref{eq:ReversedDiffusion}
can be used (after suitable discretization) to sample from $\mu$. Of course,
in order for this to be a viable strategy, we need to be able to:
$(i)$~sample from $\nu=\mu^{\bZ}_S$, and $(ii)$~compute the drift $\obF$.

A specific construction that facilitates this goal was put forward in 
\cite{song2019generative,ho2020denoising,song2021score},
which suggested to use the Ornstein-Uhlenbeck process\footnote{We could consider the more general linear drift process.
$\de  \bZ^{\slin}_s = h(s) \,\bZ^{\slin}_s\, \de s+ \sqrt{g(s)}\, \de \bB_s$
It is immediate to show that $(\bZ_s)_{s\in [0,S]}$ 
is a time-dependent rescaling of the process of 
Eq.~\eqref{eq:LinearForwardDiffusion}. Namely, we can find deterministic functions $r(s)$, $u(s)$
such that $(\bZ^{\slin}_s)_{s\ge 0}$ is distributed as $(r(s)\bZ_{u(s)})_{s\ge 0}$.}
\begin{align}
\de  \bZ_s = - \,\bZ_s\, \de s+ \, \sqrt{2} \, \de \bB_s\, .
\label{eq:LinearForwardDiffusion}
\end{align}
In other words, we set $\bF(s,\bz)  = -\bz$ and $g(s)=2$ in Eq.~\eqref{eq:FirstDiffusion}.
By integrating this equation with initial condition $\bZ_0=\bx$, we 
\begin{align}
\bZ_s \ed  e^{-s}\, \bx+\sqrt{1-e^{-2s}}\bG\, \, \;\;\;\;\;\;\; \bG \sim\normal(\bzero,\id_d)\indep \bx\, .
\end{align}
thus recovering the well known fact the  distribution $\mu^{\bZ}_s$  converges exponentially
fast to $\mu^{\bZ}_{\infty}=\normal(0,\id_n)$ (e.g. in chi-squared or Wasserstein-$2$
distance). Hence, if we choose $S$ large,
 we can approximately sample from $\mu^{\bZ}_S$. 

In order to evaluate formula \eqref{eq:ReverseDrift},  
we note that $\mu^{\bZ}_s$ is the distribution of a scaling of $\bx$
corrupted by Gaussian noise with variance $1-e^{-2s}$. 
Tweedie's formula\footnote{More commonly, Tweedie's formula is written for a standard Gaussian.
If $\bZ = \bx+\bG$, $\nabla_{\bz} \log \mu^{\bZ}_s(\bz) = \E[\bx|\bZ=\bz]-\bz$.} 
\cite{robbins1956empirical} states that 
\begin{align}
\nabla_{\bz} \log \mu^{\bZ}_s(\bz) = 
\frac{1}{1-e^{-2s}}\big\{\E[e^{-s}\bx|\bZ_s=\bz]-\bz\big\}\, 
\end{align}
It is convenient to introduce a notation for the posterior expectation of $\bx\sim\mu$
given a Gaussian observation. We define
\begin{align}
\label{eq:Denoiser}
\bm(\by;t) := \E[\bx|t\,\bx+\sqrt{t}\bG  = \by]\,,\;\;\;\;\; (\bx,\bG)\sim\mu\otimes 
\normal(0\, \id_n) \, .
\end{align}

We then apply the general formula \eqref{eq:ReverseDrift}
using $\bF(s,\bz)  = -\bz$, $g(s)=2$,  and  setting $\ft(s) = 1/(e^{2s}-1)$ with
$T=S=\infty$, we obtain
\begin{align}
\obF(t,\by) &= -\frac{1+t}{t(1+t)}\,  \by+
\frac{1}{\sqrt{t(1+t)}}\,  \, \bm\big(\sqrt{t(1+t)}\by;t\big)\, ,\label{eq:LinDrift_Basic}
\\
\og(t) & =\frac{1}{t(1+t)}\, 
.\label{eq:LinDiffusion_Basic}
\end{align}
Using this drift and diffusion coefficients,  the process \eqref{eq:ReversedDiffusion}
initialized at $\obY_0\sim\normal(0,\id_d)$ is such that 
$\obY_\infty\sim \mu$.  Explicitly
\begin{align}
\de \obY_t = -\frac{1+t}{t(1+t)}\,  \obY_t+
\frac{1}{\sqrt{t(1+t)}}\,  \, \bm\big(\sqrt{t(1+t)}  \obY_t;t\big)\, \de t+
\frac{1}{\sqrt{t(1+t)}}\, \de\bB_t\, .\label{eq:SimpleLinDiffusion}
\end{align}

Hence, stopping at an earlier time $T$ yields an approximate sample from $\mu$.

\subsection{A special stochastic localization process}
\label{sec:SL}

General stochastic localization is defined in \cite{eldan2013thin,eldan2020taming,eldan2022analysis,chen2022localization}
 as a stochastic process taking values
in the space of probability measures in $\reals^n$. At each time $t\in[0,\infty)$,
we are given a \emph{random} probability measure $\mu_t$. This process must satisfy two properties
First, as $t\to\infty$, $\mu_t$ `localizes', i.e. $\mu_t\Rightarrow \delta_{\bx}$ for a random $\bx$.
Second, it must be a martingale.

If random probability measures sound unfamiliar to the reader, there is  
a potentially simpler and equivalent\footnote{Equivalence holds under very mild measure-theoretic
conditions, e.g. that the underlying measurable space is a standard Borel space.} way to think about 
stochastic localization processes.
Sample $\bx\sim \mu$ and, at each time $t$ let $\bY_{t}$ be a noisy observation 
of $\bx$ (a random vector),  with $\bY_{t}$ becoming `more informative' as $t$ increases.
We then set $\mu_t$ to be the conditional
distribution of $\bx$ given $\bY_t$: $\mu_{t}(\bx\in \, \cdot\, )=\prob(\bx\in \, \cdot\, |\bY_{t})$.
`More informative' can be formalized in many ways, but one 
possibility is to require that, for any $t_1\le t_2$,
$\bx,\bY_{t_2},\bY_{t_1}$ forms a Markov chain (in general,
 an inhomogeneous one).
 
 We will refer to such a process $(\bY_t)_{t\ge 0}$ as to the \emph{observation process}.
A crucial observation (formalized in Section \ref{sec:GeneralSL} and illustrated in 
Section \ref{sec:Examples}) is that
\begin{quote}
\emph{$\bY_t$ does not need to take values in the same space as $\bx$.}
\end{quote}
For instance, $\bY_t$ does not need to have the same dimensions as $\bx$. 

To begin with the simplest example, consider the case in which $\bY_{t}$ is Gaussian
with 
\begin{align}
\bY_{t} = t\, \bx + \bW_{t}\, ,
\end{align}
where $(\bW_{t})_{t\ge 0}$ a standard Brownian motion. It is intuitively clear
and easy to check that $\bY_{t}$ becomes `more informative' about $\bx$ in
the technical sense given above. Roughly speaking, despite we are adding noise 
via the Brownian motion $\bW_t$, we are also increasing the signal-to-noise ratio. 

The most straightforward
way to check this formally is to write the joint distribution of $\bx,\bY_{t_1},\bY_{t_2}$.
A more elegant approach is to define $\bX_{\sigma}:= \sigma^2\bY_{1/\sigma^2}$
and noting that, by invariance properties of the Brownian motion 
$\bX_{\sigma}= \bx+\widetilde{\bW}_{\sigma^2}$ where $(\widetilde\bW_{\sigma^2})_{\sigma^2\ge 0}$
is a Brownian motion. The Markov property follows from the Markov property of the
Brownian motion.

We can now write explicitly $\mu_{t}$ using Bayes rule:
\begin{align}
\mu_{t}(\de\bx) & = \frac{1}{Z'}\mu(\de\bx)\, \exp\Big\{-\frac{1}{2t}\|\bY_{t}-t\bx\|_2^2\Big\}\nonumber\\
& = \frac{1}{Z}\mu(\de\bx)\, e^{\<\bY_{t},\bx\>-\frac{t}{2}\|\bx\|^2}\, , \label{eq:SimpleTilting}
\end{align}
where $Z$ and $Z'$ are normalizations, depending on $\bY_t$, $t$.
In other words, $\mu_{t}$ is a random tilt of $\mu$. 

Notice that $\mu_{t=0}=\mu$ is the original distribution, and 
$\mu_{t}\Rightarrow \delta_{\bx}$ as $t\to\infty$. The tilting factor
completely localizes the measure on $\bx_*$. 
Hence if we can generate the stochastic process $(\mu_{t})_{t\ge 0}$,
we can also sample from $\mu$. For instance, we can compute the 
baricenter $\bm_t := \int \!\obx\, \mu_{t}(\de \obx)$ and note that $\bm_t\to\bx$
as $t\to\infty$, with
$\bx\sim\mu$. 

At first sight, this strategy appears problematic for two reasons.
First, the definition of $\mu_{t}$ in Eq.~\eqref{eq:SimpleTilting}
depends on $\bY_{t}$, which is itself defined in terms of $\bx\sim \mu$.
It might seem that we need to sample $\bx$ to begin with.
Second, $(\mu_{t})_{t\ge 0}$ is ---in general--- 
a stochastic process taking values in an infinite-dimensional space. Even if we consider
a discrete setting, e.g. $\bx\in\{+1,-1\}^n$, $(\mu_{t})_{t\ge 0}$ takes
place in exponentially many dimensions. 

Both of these issues are taken care of by the following classical fact
in the theory of stochastic processes \cite[Section 7.4]{liptser1977statistics}.
\begin{proposition}\label{propo:SLESDE}
Assume $\mu$ has finite second moment. Then, $(\bY_{t})_{t\ge 0}$ is the unique solution
of the following stochastic
differential equation (with initial condition $\bY_0=\bfzero$)
\begin{align}
\de\bY_{t} = \bm(\bY_{t};t)\, \de t+ \de\bB_{t}\, .
\label{eq:SLESDE}
\end{align}
Here $(\bB_{t})_{t\ge 0}$ is a standard Brownian motion and $\bm(\by;t)$
is the conditional expectation defined in Eq.~\eqref{eq:Denoiser}, i.e.
\begin{align}
\bm(\by;t) := \E[\bx|t\,\bx+\sqrt{t}\bG  = \by]\,,\;\;\;\;\; (\bx,\bG)\sim\mu\otimes \normal(0\, \id_n)\, .
\end{align}
\end{proposition}
We found therefore a different way to construct a diffusions-based
algorithm  to sample from  $\mu$. In a nutshell, we discretize the SDE of Eq.~\eqref{eq:SLESDE}, 
for $t \in [0,T]$ for some large $T$,  and then use $\bY_T/T$ as an approximate sample
from $\mu$. Alternatively, we can output $\bm(\bY_T;T)$. 

The SDE of Eq.~\eqref{eq:SLESDE} looks tantalizingly similar to the
one of Eq.~\eqref{eq:SimpleLinDiffusion}, despite the fact that 
we derived them by quite different arguments (time-reversal in the first case and stochastic localization
in the second). A simple calculation shows that they in fact coincide after the 
change of variables
\begin{align}
\bY_{t} =\sqrt{t(1+t)}\, \obY_t\, .
\end{align}

The rest of this paper is devoted to generalizing this construction.
In particular, in Section \ref{sec:GeneralSL} we describe the natural generalization of this
construction to general stochastic localization processes.

\section{Literature overview}
\label{sec:Literature}

The point of view followed here is closely related to 
\cite{el2022sampling}, whose approach was further developed in \cite{montanari2023posterior}.
Our main objective is to show that this approach provides a unifying viewpoint on
recent developments in the deep learning literature and allows for some new insights.

Over the last few years, a
 number of generalizations and variants on the original construction of 
denoising diffusions in
\cite{song2019generative,ho2020denoising,song2021score} have been 
proposed and studied. An incomplete summary includes:
\begin{itemize}
\item The discrete-time version of the Ornstein-Uhlenbeck process 
\eqref{eq:LinearForwardDiffusion} is a linear autoregressive process that is driven by Gaussian noise. 
The idea of adding some form of non-Gaussian noise was proposed in \cite{nachmani2021non}.
(See also  \cite{deasy2021heavy} for a related idea).
\item Discrete analogous of diffusion processes (for vectors $\bx$ with values in $\{1,\dots,k\}$)
 were introduced in \cite{sohl2015deep,hoogeboom2021argmax}. In the forward process, each coordinate
 of $\bx$ is flipped to a uniformly random one independently at a certain rate. The reverse 
 process requires to estimate the conditional expectation of the vector $\bx$
 given the current state.
 \item More general constructions of discrete diffusions were introduced in
 \cite{austin2021structured} which considered independent coordinate flipping with 
 arbitrary (time-dependent) transition rates. The use of absorbing states allows to incorporate 
 examples in which some coordinates are masked.
\item  An even more general framework for discrete diffusions was proposed in \cite{campbell2022continuous}.
 The authors consider a forward process that is a general continuous-time Markov process 
 with discrete state space and construct the sampling process, again by time reversal.
A variant of this approach is developed in \cite{sun2022score}. 
\item A different approach towards treating discrete variables was advocated in 
\cite{chen2022analog}. In a first step, \cite{chen2022analog} reduces the problem of sampling 
discrete vectors to the one of sampling binary vectors (representing discrete variables as 
binary strings). 
Unlike earlier works, these authors do not modify the forward process in 
the diffusion (and hence do not modify the reverse process either). They obtain
discrete samples by rounding. 
\item Further recent work on diffusions in discrete or constrained domains includes
\cite{meng2022concrete,vignac2022digress,ye2022first,liu2023learning}.
\end{itemize}
In summary, the time-reversal approach has been extended to encompass a broad array of
sampling schemes. This generality is not surprising, since any sampling procedure is the time reversal
of a process that is can be viewed as a `noising' of the original distribution.

As we will see, the stochastic localization approach is equally general, but also naturally
suggests a different palette of sampling schemes.

\rev{We also note that a complementary viewpoint for generative modeling
is provided by continuous time probability flows, see e.g.
\cite{albergo2022building,lipman2023flow}.
The idea is to construct a flow $\phi_t:\reals^n:\reals^n$ which
pushes forward a `simple' distribution $\nu_0$ (e.g. standard normal)
$\nu_t=(\phi_t)_{\#}\nu_0$ is such that $\nu_1=\mu$. 
While standard constructions in this setting are deterministic, 
\cite{albergo2023stochastic} defined stochastic flows which encompass diffusions as well.}

\section{General stochastic localization sampling}
\label{sec:GeneralSL}

Given  $\bx\sim \mu$ a random variable in $\reals^n$, we construct a sequence of 
random vectors\footnote{In all of our examples,
the vectors $\bY_{t}$ belongs to a finite dimensional space $\reals^m$ (possibly $m\neq n$), but
this needs not to be the case. In fact, we can consider them as elements in a 
Borel-equivalent space.} $(\bY_{t})_{t\in I}$ indexed by $I\subseteq[0,\infty]$
(typically $I$ will be an interval, but this does not need to be the case).

We assume that $\bY_{t}$ is increasingly 
more informative about $\bX$ as $t$ increases, as formalized by the following definition.
\begin{definition}\label{def:GeneralDef}
\rev{We say that the process  $(\bY_{t})_{t\in I}$ is \emph{an observation process}
with respect to $\bx$  if the following conditions hold:}
\begin{enumerate}
\item[$(i)$] \rev{For each integer $k$, and for each $t_1<t_2<\dots<t_k\in I$,
sequence of random variables $\bx$, $\bY_{t_k}$, $\bY_{t_{k-1}}$,\dots,$\bY_{t_1}$ 
forms a 
Markov chain. Namely the conditional distribution $\prob(\bY_{t_{i-1}}\in\,\cdot\,|
\bx,\bY_{t_{i}},\dots \bY_{t_k})$ coincides with the probability distribution 
$\prob(\bY_{t_{i-1}}\in\,\cdot\,|\bY_{t_i})$.
\item[$(ii)$] The process $(\bY_t:t\in I)$ is complete:
observing the whole path $(\bY_{t})_{t\in I}$ gives complete information about $\bx$.
Formally, any (measurable) set $A\subseteq \reals^n$:}
\begin{align}
\prob(\bx\in A|\bY_{t},t\in I) \in \{0,1\}\, .
\end{align}
\end{enumerate}

Given such a process $\bY_{\cdot}:=(\bY_{t})_{t\in I}$, we define the 
\emph{stochastic localization process} (or \emph{scheme}) to be the sequence of posteriors
$(\mu_{t})_{t\in I}$:
\begin{align}
\mu_{t}(\,\cdot\,) :=  \prob(\bx\in\,\cdot\,|\bY_{t})\, .\label{eq:DefinitionSL}
\end{align}
\end{definition}
We can interpret $\bY_{t}$ as noisy observations or measurements of the underlying random variable
$\bx$, which become less noisy as $t$ increases. The notion of ordering 
among conditional distributions $\prob(\bY_{t_{k}}\in\,\cdot\,|\bX)$, \dots, 
$\prob(\bY_{t_{1}}\in\,\cdot\,|\bX)$ is the well-known `Blackwell ordering'
between statistical experiments \cite{blackwell1953equivalent,le1996comparison}.
The ordering of  random variables  $\bY_{t_k}$, $\bY_{t_{k-1}}$,\dots,$\bY_{t_1}$ 
is also common in the information theory 
literature, and referred to as   `ordering by physical degradation'
\cite{bergmans1973random,RiU08}.
We will henceforth interpret $t$ either as 
`time' or as `signal strength.'

\begin{remark}[Relation to earlier definitions, I]
As we will see in detail below, Definition \ref{def:GeneralDef} encompasses 
various proposals in the machine learning literature, e.g.
\cite{song2019generative,ho2020denoising,song2021score,hoogeboom2021argmax,
austin2021structured,campbell2022continuous,sun2022score}. 
It is worth repeating that, according to Definition \ref{def:GeneralDef}
observations $\bY_t$ can take values in a space different from the one of $\bx_*$,
and in different spaces for different $t$, while this possibility was not 
exploited in earlier works.

Because of this feature, we can in fact construct a scheme satisfying 
Definition \ref{def:GeneralDef} given \emph{any sequence of random variables}
$(\bZ_{t})_{t\in I}$: it is sufficient to define 
$\bY_t = (\bZ_s: s\le t)$.
\end{remark}
\begin{remark}[Relation to earlier definitions, II]
\rev{As already mentioned, the definition if stochastic localization process as a measure valued martingale
$(\mu_t)_{t\ge 0}$ given in Section \eqref{sec:SL} is equivalent to the one of 
Definition \ref{def:GeneralDef}.}

\rev{To see this, note that the measure-valued stochastic process  \eqref{eq:DefinitionSL}
 is a Doob martingale \cite{williams1991probability}. Namely, for any 
 measurable set $A$, $\mu_t(A) = \E(\bfone_A(\bx)|\cF^{\bY}_t)$ where
$\cF^{\bY}_{t}: = \sigma(\{\bY_{t'}:\; t'\le t\})$ is
the $\sigma$-algebra generated by observations up to time $t$. 
Hence $\E[\mu_t(A)|\cF^{\bY}_{s}] = \mu_s(A)$ for any $s\le t$. The completeness condition
further implies $\mu_t\Rightarrow \delta_{\bx}$ (see below).}

\rev{Viceversa, let $\mu_t$ be any measure valued martingale
with respect to a filtration $(\cF_t)$ such that $\mu_t\Rightarrow \delta_{\bx}$ 
with $\bx\sim \mu$ as per the definition of Section \eqref{sec:SL}. 
Then, by Levy convergence theorem, $(\mu_t(A))_{A\in \cA}\to 
(\mu_\infty(A))_{A\in \cA}$ almost surely for any countable collection $\cA$,
whence $\mu_t(A) = \E[\mu_\infty(A)|\cF_t]$ for every $A\in\cA$, and hence for every
measurable set (assuming the underlying space is standard Borel). 
Finally, since $\mu_\infty(A) = \bfone_{\bx\in A}$ for a random $\bx\sim\mu$, we have
 $\mu_t(A)=\prob(\bx\in A|\cF_t)$. We can choose $\tilde\bY_t$ any set of random variables that generates
 the $\sigma$-algebra $\cF_t$ (this always exists  for standard Borel spaces),
 and $\bY_t = (\tbY_s:\, s\le t)$.}
 
  We refer also
 to \cite{el2022information} where this connection was explored in a special case.
\end{remark}

\begin{remark}[Completeness]
\rev{We note that the observation that the process is complete can be restated as
\begin{align}
\mu_{\infty}(A)  \in \{0,1\}\, ,\;\;\; \forall A\in \cB_{\reals^n}\, ,
\end{align}
where $\cB_{\reals^n}$ is the Borel $\sigma$-algebra.
The notation $\mu_{\infty}$ is justified since $\lim_{t\to\infty}\mu_{t}(A)=\mu_{\infty}(A)$
almost surely
(by Levy's martingale convergence theorem).
Since $\mu_{\infty}(A) \in \{0,1\}$ for all $A$, it follows that $\mu_{\infty}(A) =\bfone_{\bx\in A}$,
 $\bx\sim \mu$.}
 
 \rev{Clearly, A sufficient condition for the  requirement to be satisfied is that there exists 
 a measurable function $f$ such that $\lim_{t\to T}f(\bY_t;t)=\bx$ or a function $h$ such that 
 $h(\bY_T)=\bx$.}
\end{remark}

We now make a trivial, yet important remark:
\begin{remark}\label{BasicFact}
\emph{Since $\bx$, $\bY_{t_k}$, $\bY_{t_{k-1}}$,\dots,$\bY_{t_1}, \bY_0$ 
forms a Markov Chain, so is the reverse sequence
$\bY_0,\bY_{t_1}$, $\bY_{t_{2}}$,\dots,$\bY_{t_k}$,
$\bx$.}
\end{remark}
There exists therefore a transition probability $\sP_{t,t'}(\by|A )=
\prob(\bY_{t'}\in A|\bY_{t}=\by)$ indexed by $t,t'\in I\cup \infty$
(with $\bY_{\infty}:=\bx$).
This provides the blueprint for constructing a general sampling scheme:
\begin{enumerate}
\item
Discretize (if necessary) the time index set
to $I_{m}:= (t_0 = 0,t_1,\dots,t_m)$.
\item  Construct approximate probability kernels 
$\hsP_{t_k,t_{k+1}}(\by_k|\,\cdot\, ) \approx \sP_{t_k,t_{k+1}}(\by_k|\,\cdot\, )$.
\item  For each $k\in \{0\,\dots,m\}$,
sample
\begin{align}
\by_{k+1} \sim  \hsP_{t_k,t_{k+1}}(\by_k|\,\cdot\, )\, .
\end{align}
\end{enumerate}
Of course this procedure yields an algorithm only if the transition probability 
$\sP_{t,t'}(\by|\,\cdot\,)$ can be approximated efficiently 
for $t'$ close to $t$. In the next sections we will discuss a few 
special cases.
%
%
\section{A dozen examples of sampling schemes}
\label{sec:Examples}

\subsection{The isotropic Gaussian process}
\label{sec:Isotropic}

This is simply the construction of Section \ref{sec:SL}, whose definition we copy here
for the readers' convenience (recall that $\bW$ is a standard Brownian motion)
\begin{align}
\bY_{t} = t \bx_* + \bW_{t}\, ,\;\;\;\; t \in I=[0,\infty)\, 
\end{align}
As we discussed there, this process $(\bY_{t})_{t\in I}$ satisfies the conditions
of an observation process.
The SDE \eqref{eq:SLESDE}, namely
\begin{align}
\de\bY_t &= \bm(\bY_t;t)\, \de t + \de \bB_t\, ,\label{eq:Isotropic}
\end{align}
confirms that it  is a Markov process, 
as anticipated by Remark \ref{BasicFact}. 
An approximate transition probability can be constructed 
by a Euler discretization of this SDE. Namely, given a mesh
$I_{m}:= (t_0 = 0,t_1,\dots,t_m)$, we compute
\begin{align}
\hbY_{k+1} = \hbY_{k+1} + \bm(\hbY_{k};t_k)\, \delta_k+ \bG_k\sqrt{\delta_k}\,,
\;\;\;\;\;\; \delta_k := t_{k+1}-t_k\, ,
\end{align}
where $(\bG_k)_{k\ge 1}\sim_{i.i.d.}\normal(0,\id_n)$.
The corresponding approximate transition probability is 
\begin{align}
\hsP_{t_k,t_{k+1}}(\by_k|\de\by_{k+1})
= \frac{1}{(2\pi \delta_k)^{n/2}}\, \exp\Big\{-\frac{1}{2\delta_k}\big\|
\by_{k+1}-\by_k-\delta_k\bm(\by_{k};t_k)\big\|_2^2\Big\}
\,\de\by_{k+1}\, .
\end{align}
Improved discretizations are given in   \cite{song2020denoising,karras2022elucidating}.

\begin{remark}
\rev{We observe that (under certain conditions) the SDE \eqref{eq:Isotropic} 
implies an equivalent one for 
$\bm_t$. Namely, differentiating $\bm(\bY_t,t)$, and using Ito's formula we get
(assuming $(\by,t)\mapsto\bm(\by,t)$ is twice differentiable in $\by$ and once in $t$)}
\begin{align}\label{eq:FirstEqForM}
\de\bm_t = \bD_{\bY}\bm(\bY_t,t) \big[\bm_t\de t+\de\bB_t\big]+\frac{1}{2}
\bD^{\otimes 2}_{\bY}\bm(\bY_t,t) \{\id\} \de t + \frac{\partial \bm}{\partial t}(\bY_t,t)\de t\, ,
\end{align}
\rev{where $\bD_{\bY}$ is the differential operator and, for
$\bT\in (\reals^{n})^{\otimes 3}$, $\bA\in\reals^{n\times n}$,
we defined  $\bT\{\bA\}\in\reals^n$ via
$\bT\{\bA\}_i :=\sum_{j_1,j_2=1}^nT_{ij_1j_2}A_{j_1j_2}$.}

\rev{Using the expression \eqref{eq:SimpleTilting} it is easy to compute}
\begin{align}
\bD_{\bY}\bm(\bY_t,t) &=   \E\big[(\bx-\bm_t)^{\otimes 2}\big|\bY_t\big]=:\bQ(\bY_t,t)\, ,\\
\bD^{\otimes 2}_{\bY}\bm(\bY_t,t) &=   \E\big[(\bx-\bm_t)^{\otimes 3}\big|\bY_t\big]\, ,\\
 \frac{\partial \bm}{\partial t}(\bY_t,t) & = -\frac{1}{2}
  \E\big[(\bx-\bm_t)\|\bx\|^2_2\big|\bY_t\big]\, . 
\end{align}
\rev{In particular, since $\mu_t$
has finite third moment for $t>0$, the differentiability assumptions 
to derive Eq.~\eqref{eq:FirstEqForM} are justified. Substituting in Eq.~\eqref{eq:FirstEqForM}
yields}
\begin{align}\label{eq:SecondEqForM}
\de\bm_t = \bQ(\bY_t,t)\de\bB_t\, .
\end{align}
\rev{Notice that all terms with non-zero expectation vanished. This is expected since
$\bm_t = \int \bx \, \mu_t(\de\bx)$ and $\mu_t$ is a martingale, whence $\bm_t$ is a martingale.
Equation \eqref{eq:SecondEqForM} is does not provide an evolution for $\bm_t$ unless
we can invert the map $\by\mapsto\bm(\by,t)$. Invertibility holds whenever $\bQ(\by,t)\succ \bzero$
strictly for all $\by$.}
\end{remark}

\subsection{The anisotropic Gaussian process}
\label{sec:Anisotropic}

An obvious generalization is to allow for non-identity covariance.
Namely, for $\bQ:[0,\infty)\to \S_+(n)$ (the cone of $n\times n$ positive semidefinite
matrices) we define
\begin{align}
\bY_{t} = \int_0^t \bQ(s)\bx\,  \de s+ \int_0^t\bQ(s)^{1/2}\, \de\bW_{s}\, ,\;\;\;\; t \in I=[0,\infty)
\, .
\end{align}
This satisfies the SDE
\begin{align}
\de\bY_t &= \bQ(t) \bm(\bY_t;\bOmega(t))\, \de t + \bQ(t)^{1/2} \de \bB_t\, ,
\;\;\;\;\;\;
\bOmega(t):= \int_0^t\bQ(s)\, \de s\, .\label{eq:AnisotropicSDE}
\end{align}
where, for $\bOmega\in\S_+(n)$, we let 
\begin{align}
\bm(\by;\bOmega) := \E[\bx|\bOmega\bx+\bOmega^{1/2}\bG =\by]\, . \label{eq:CondExp}
\end{align}
The following Euler discretization can be used to sample:
\begin{align}
\hbY_{k+1} = \hbY_{k+1} + \bQ(t_k)\bm(\hbY_{k};\bOmega(t_k))\, \delta_k+ 
\bQ(t_k)^{1/2} \bG_k\sqrt{\delta_k}\,,
\;\;\;\;\;\; \delta_k := t_{k+1}-t_k\, ,
\end{align}

\subsection{The erasure process}
\label{sec:Erasure}

For each $i\in [n]$, we let $T_i$ be an independent random variable $T_i\sim \Unif([0,1])$
and set
\begin{align}
Y_{t,i} = \begin{cases}
x_i & \mbox{ if }t\ge T_i,\\
\star & \mbox{ if }t< T_i.
\end{cases}
\end{align}
In this case $I=[0,1]$ and $\bY_{t}\in (\reals\cup \{\star\})^n$ is obtained by `erasing'
independently each coordinate of $\bx$ with probability $1-t$. 

The associated sampling algorithm is the standard sequential sampling
procedure. 
Namely, for $t\in\{1,\dots, n\}$, we sample 
\begin{align}
x_{i(t)} \sim  \mu\big(x_{i(t)}\in \,\cdot\, \big|x_{i(1)},\dots, x_{i(t-1)}\big)\, .
\end{align}
 
 Of course this process can be modified by choosing the revealing times $T_i$
 to be a deterministic sequence. In that way we obtain sequential sampling
 with  $i(1)$, \dots $i(n)$ any predefined order.
 
 \subsection{The binary symmetric  process}
 \label{sec:BSC}

We next give a (continuous-time) reformulation of the binary sampling scheme of 
\cite{sohl2015deep,hoogeboom2021argmax}.

We assume $\bx\in\{+1,-1\}^n$, $\bx\sim \mu$ and set (with $\odot$ the Hadamard product)
\begin{align}
\bY_t = \bx\odot \bZ_t\, , 
\end{align}
where $(\bZ_t)_{t\in I}$, $I=[0,1]$ is a suitable noise process
 taking values in $\{+1,-1\}^n$.
 Before defining the process, we highlight that 
 $Z_{t,i}\in\{+1,-1\}$ with $\E Z_{t,i} = \prob(T_i<t) =t$. Equivalently
\begin{align*}
\prob(Z_{t,i} = +1) = 1- \prob(Z_{t,i} = -1)  = \frac{1+t}{2}\, .\label{eq:MarginalBernoulli}
\end{align*}
In particular $\bY_0$ is uniformly random in $\{+1,-1\}^n$, and $\bY_1=\bx$.
In other words, the signal-to-noise ratio becomes larger as $t$ grows from $0$ to
$1$. 
  
Informally, the process $(\bZ_t)_{t\in I}$ is defined by the fact that its coordinates
are independent and identically distributed with each coordinate defined as follows.
Start with $Z_{i,1}=+1$, and generate $Z_{i,t}$ proceeding 
backward in time. For each interval $(t-\delta,t]$ replace $Z_{t,i}$
with a fresh random variable independent of the $(Z_{s,i})_{s\ge t}$ with probability
$\delta/t +o(\delta)$. 
  
It is clear from this definition that, for any $t_1<t_2<\dots<t_k$,
$\bZ_{t_k},\bZ_{t_{k-1}}\dots, \bZ_{t_1}$ forms  a Markov chain, and hence so does
$\bx,\bY_{t_k},\bY_{t_{k-1}}\dots, \bY_{t_1}$.

Further, calling $T_{i,1}$ the first time (proceeding backward from $1$) at which 
$Z_{i,t}$ is resampled, if follows from the definition that 
$\prob(T_{i,1}<t) =\exp(-\int_{[t,1]}s^{-1}\de s)=t$.
In other words, $T_{i,1}$ is uniformly random in $[0,1]$,
whence Eq.~\eqref{eq:MarginalBernoulli} immediately follows.
\begin{remark}
This remark provides a more rigorous definition of the process $(\bZ_t)_{t\in [0,1]}$.
Let, independently for each $i\le n$, $T_{i,1}>T_{i,2}>\cdots$ be the arrival times of a Poisson 
process with density $\nu(\de t) = \bfone_{[0,1]}(t)t^{-1}\de t$,
and let $\{R_{i,\ell}\}_{\ell\ge 1}$,  $R_{i,\ell}\sim \Unif(\{+1,-1\})$.
Further, define $R_{i,0} = +1$, $T_{i,0}=1$. 
We then set
\begin{align}
Z_{t,i} = R_{i,\ell}\;\;\Leftrightarrow \;\; 
T_{i,\ell}\ge t > T_{i,\ell+1}\, .
\end{align}
\end{remark}

\begin{remark}
An equivalent definition is as follows. Let $(\bX_s)_{s\ge 0}$
be continuous random walk in the hypercube started at $\bX_0=\bx$. Namely,
within any interval $[s,s+\delta)$, with probability $\delta+o(\delta)$,
coordinate $i$ is replaced independently from the others
by a uniformly random variable in $\{+1,-1\}$.

We then set $\bY_t = \bX_{\log(1/t)}$, for $t\in(0,1]$.
\end{remark}

In agreement with our general Remark \ref{BasicFact}, the process $(\bY_t)_{t\in[0,1]}$ is also 
Markov forward in time. Indeed it is a continuous-time Markov
chain initialized at $\bY_0\sim \Unif(\{+1,-1\}^n)$. In the interval $[t,t+\delta)$, coordinate
$i$ of $\bY_t$ flips, independently of the others with probability 
(here $\by^{(i)}$ is defined by $y^{(i)}_i=-y_i$ and  $y^{(i)}_j=y_j$ for $j\in [n]\setminus i$):
\begin{align}
\prob(\bY_{t+\delta} = \by^{(i)}|\bY_{t} = \by) = p_i(\by;t)\,\delta +o(\delta)\, .
\end{align}
The transition rates are given by
\begin{align}
p_i(\by;t) &= \frac{1+t^2}{2 t(1-t^2)}-\frac{1}{1-t^2}y_i\, m_i(t;\by)\, ,\\
m_i(t;\by)& := \E[x_i|\bY_t=\by]\, .
\end{align}

\subsection{The symmetric process}
\label{sec:SC}

We can generalize the previous process to the case of a $q$-ary alphabet
$x_i\in [q]=\{1,\dots, q\}$, the result being equivalent to the process 
introduced in \cite{hoogeboom2021argmax}.
As before, $I=[0,1]$ and, for each $i\in [n]$, we let $\{T_{i,\ell}\}_{\ell\ge 1}$ be an independent 
Poisson point process with rate $\nu(\de t) = \bfone_{[0,1]}(t)t^{-1}\de t$,
and $\{R_{i,\ell}\}_{\ell\ge 1}$ an independent sequence of random variables $R_{i,\ell}\sim \Unif([q])$.
We then set
\begin{align}
Y_{t,i} = \begin{cases}
x_i & \mbox{ if } T_{i,1}<t_i\le 1,\\
R_{i,\ell}& \mbox{ if }T_{i,\ell+1} < t\le T_{i,\ell}.
\end{cases}
\end{align}
As noted before $T_{i,1}\sim \Unif([0,1])$ and therefore $(Y_{t,i})_{i\le n}$
are conditionally independent given $\bx$, with
\begin{align}
\prob(Y_{t,i}=y |\bx) = \begin{cases}
(1+(q-1)t)/q & \mbox{ if } y=x_i\, ,\\
(1-t)/q& \mbox{ if } y\neq x_i\, .\\
\end{cases}
\end{align}

Again, by the general Remark \ref{BasicFact}, the process $(\bY_t)_{t\in[0,1]}$ is a 
Markov forward in time. For $\by\in [q]^n$, $z\in [q]\setminus\{y_i\}$,
let $\by^{(i,z)})_j = y_j$ if $j\neq i$, and $\by^{(i,z)})_i = z$. Then 
the transition rates are given by 
\begin{align}
\prob(\bY_{t+\delta} = \by^{(i,z)}|\bY_{t} = \by) &= p_i(\by, z;t)\,\delta +o(\delta)\, ,\\
\prob(\bY_{t+\delta} = \by|\bY_{t} = \by) &= 1-\sum_{z\in [q]\setminus\{y_i\}}p_i(\by, z;t)\,\delta +o(\delta)\, ,
\end{align}
where
\begin{align}
p_i(\by,z;t) &= \frac{1}{qt}+\frac{1}{1-t} b_i(\by,z;t)-\frac{1}{1+(q-1)t} b_i(\by,y_i;t)\, ,\\
b_i(\by,z;t)& := \prob(x_i=z|\bY_t=\by)\, .
\end{align}

\subsection{The linear observation process}
\label{sec:Linear}

For a fixed matrix $\bA\in\reals^{m\times n}$ and an $m$-dimensional 
standard Brownian motion $(\bB_t)_{t\ge 0}$, we observe \cite{montanari2023posterior}
\begin{align}
\bY_t = t\, \bA\bx+\bB_t\, ,
\end{align}
where $\bx\sim\mu$.
By the same argument in Section \ref{sec:Isotropic}, $\bY_t$ satisfies   the SDE
\begin{align}
\de\bY_t = \bm_{\bA}(\bY_t;t)\, \de t +\de\bB_t\, ,\label{eq:ADiffusion}
\end{align}
where $\bm_{\bA}(\bY_t;t)$ is the minimum mean square error estimator of $\bA\bx$
\begin{align}
\bm_{\bA}(\by;t) =\E\big\{\bA\bx\big| t\bA\bx+\sqrt{t}\bG =\by\big\}\, .
\end{align} 

By discretizing Eq.~\eqref{eq:ADiffusion} as in Section \ref{sec:Isotropic},
we can sample $\obY_{\infty}$, an approximation of $\bY_{\infty} =\bA\bx$. Note that, unlike for
the original construction, once the diffusion process is terminated, we still
need to generate $\bx$ from the $\obY_{\infty} =\bA\bx+{\rm error}$, in a way that is robust to 
sampling errors. 
Two examples in which this can be done easily:
\begin{itemize}
\item $\bA$ has full column rank (in particular, $m\ge n$).
Then we can output $\bx := \bA^{\dagger}\obY_{\infty}$ (with 
$ \bA^{\dagger}:=(\bA^{\sT}\bA)^{-1}\bA^{\sT}\obY_{\infty}$ the pseudoinverse of $\bA$).
\item $\bA$ does not have full column rank (for instance, $m< n$),
but $\bx$ is structured, for instance is sparse. In this case,
we can find $\bx$ by using compressed sensing techniques, e.g. by solving
\begin{align}
\mbox{minimize} & \;\;\; \|\bx\|_1\, ,\\
\mbox{subj. to} & \;\;\; \bA\bx = \obY_{\infty}\, .
\end{align}
\end{itemize}

An alternative, construction would be instead to add noisy linear measurements.
In this case time $t\in \naturals$ is discrete, $\bY_t\in \reals^t$ and
\begin{align}
\bY_{t} & = (Y_1,\dots,Y_t)\, ,\\
Y_t &= \<\ba_t,\btheta\>+ \eps_t\, ,\;\;\; \eps_t\sim\normal(0,\sigma^2)\, .
\end{align}
Here $\ba_1,\ba_2,\dots \in\reals^n$ is a sequence of vectors generated with a 
predefined process (either deterministic or random).
The transition probability of this Markov chain is given by
\begin{align}
\prob(Y_{t+1}\le a|\bY_{t}=\by) = 
\int\! \Phi\Big(\frac{a-s}{\sigma}\Big)\, \nu_t(\de s|\bY_{t}=\by) 
\end{align}
where $\Phi(u) :=\int_{-\infty}^u\exp(-v^2/2)/\sqrt{2\pi}\, \de v$ is the standard Gaussian 
distribution and $\nu_t(\,\cdot\,|\bY_{t}=\by)$ is the conditional law of
$\<\ba_{t+1},\bx\>$ given $\bY_t=\by$.

\subsection{The information percolation process}
\label{sec:Percolation}

Let $\bx\in\integers^{m\times n}$ be a grayscale image and 
$G_{m,n}=(V_{m,n},E_{m,n})$ be the two-dimensional grid with vertex set $\{0,\dots,m\}\times\{0,\dots,n\}$.
For each edge $e\in E_{m,n}$, choose a direction arbitrarily: $e=(o,t)$,
and further order the edge set arbitrarily: $E_{m,n} = (e(1),\dots,e(N) )$,
$e_{\ell}= (o(\ell), t(\ell))$, 
$N= 2mn+m+n$. Let 
\begin{align}
  \bY(\ell) = \big(x_{t(1)}-x_{o(1)},\dots, x_{t(\ell)}-x_{o(\ell)}, \star,\dots,\star \big)
\end{align}
In words, at time $\ell$, we revealed the difference of values
along the first $\ell$ edges.
It is easy to check that this satisfies the conditions of our general construction,
indeed it is a simple change of variables  of the erasure process of Section \ref{sec:Erasure}.

The transition probabilities are easy to compute
\begin{align}
\prob\big( \bY(\ell)_{\ell+1} = y \big|\bY(\ell)\big) = 
\prob\big(x_{t(\ell+1)}-x_{o(\ell+1)} = y\big|\bY(\ell)\big)\, .
\end{align}
In other words, at each step, one needs to compute the conditional distribution of 
$(x_{t(\ell+1)}-x_{o(\ell+1)})$ given the information graph revealed thus far.

\subsection{The Poisson observation process}
\label{sec:Poisson}

In this case, we assume that $\bx\in\reals_{\ge 0}^{n}$ is non-negative,
and let $\bY_t\in\naturals^n$, with coordinates conditionally independent given 
$\bx$, and $(Y_{t,k})_{t\ge 0}$ for each $k$ a Poisson Point Process ($\PPP$)
of rate $x_k$
\begin{align}
(Y_{t,k})_{t\ge 0} \big|_{\bx}\sim \PPP(x_k\de t)\, .
\end{align}
Informally, $Y_{0,k}=0$ and $Y_{t,k}$ is incremented by one in the
interval $[t,t+\de t)$ independently with probability $x_k\, \de t$.   
In particular, for each $k$, $t$, $Y_{t,k}\sim \Poisson(t x_k)$

The transition probabilities are given by
\begin{align}
\prob\big(\bY_{t+\delta} &= \by \big|\bY_t=\by\big)  = 1-\delta \sum_{k=1}^nm_k(t;\by) +o(\delta)\, ,\\
\prob\big(Y_{t+\delta,k} &= y_k+1\big|\bY_t=\by\big)  = \delta\, m_k(t;\by) +o(\delta)\, .
\end{align}
where, as before, $m_k(t;\by):= \E[X_k|\bY_t=\by]$. 

\subsection{The half-space process}
\label{sec:Rectangle}

Let $\bx\in\reals^n$ and $\{\cH_\ell\}_{\ell\ge 1}$ be a sequence of half spaces
in $\reals^n$.
Namely, $\cH_k:=\{\bz\in\reals^n: \<\ba_k,\bz\>\ge b_k\}$, for some $\ba_k\in\reals^n$,
$b_k\in\reals$. 
For $\ell\ge 0$, we let
\begin{align}
\bY_\ell = \big(\bfone_{\bx\in\cH_1},\dots, \bfone_{\bx\in\cH_\ell}\big)\, ,.
\end{align}
Note that at step $\ell$, $\bY_{\ell}$ is a binary vector of length $\ell$.

\subsection{Revisiting reverse diffusions}
\label{sec:Reverse}

\rev{In Section \ref{sec:FirstDiffusions}, we introduced the general 
diffusion (noising) process \eqref{eq:FirstDiffusion} and explained that its
time reversal is also a diffusion process \eqref{eq:ReversedDiffusion} which can be 
used for sampling (provided it can be approximated numerically).
We then showed that, if the forward process is a Ornstein-Uhlenbeck process,
of Eq.~\eqref{eq:LinearForwardDiffusion}, then its time reversal is equivalent to the 
stochastic localization process of Section \ref{sec:Isotropic}.}

\rev{This connection can be generalized and indeed any `reverse diffusion' as given in
\eqref{eq:ReversedDiffusion} is equivalent to an instance of a stochastic
localization process as introduced in Section \ref{sec:GeneralSL}. The fundamental
reason very simple and general. For $t\in [0,T]$ we  define
\begin{align}
\bY_{t} = \big(\bZ_s: \, s\in [\fs(t),S]\big)\, .\label{eq:ReverseDiffusion}
\end{align}
Then, it is easy to check that $(\bY_t)_{t\in [0,T]}$ is an observation process.
Indeed  $\bx, \bY_{t_k},\dots,\bY_{t_1}$ form a Markov chain for any $t_1\le t_2\le \dots\le t_k$,
because $\bY_{t_{i}}$ can be computed from $\bY_{t_{i+1}}$ by `deleting' 
the portion $(\bZ_s: \, s\in [\fs(t_{i+1}),\fs(t_{i})))$. Further 
$\bx$ is obviously measurable on $\bY_S = \big(\bZ_s: \, s\in [0,S]\big)$. Hence 
$\mu_S = \delta_{\bx}$, $\bx\sim\mu$.}

\rev{At first sight, it is not obvious that the observation process is equivalent 
to the reverse diffusion \eqref{eq:ReverseDiffusion}. However, also this can be explained
without calculation. Notice that $\bY_t$ defined in Eq.~\eqref{eq:ReverseDiffusion}
is a path that we can view as indexed by $s\in [\fs(t),S]$ or, equivalently, 
$\theta\in [0,T]$: we denote this trajectory by $\bY_{t}(\theta)$, $\theta\in [0,t]$
where $\bY_{t}(\theta)= \bZ_{\fs(\theta)}$.}

\rev{In particular, $\bY_t(t)=\obY_t$ and therefore $\obY_t$ is a measurable
function of $\bY_t$. Vice versa, given $\bY_{t} =(\obY_\theta:\; \theta\in [0,t])$,
we can generate $\bY_{t+\delta}$ by integrating the SDE   \eqref{eq:ReversedDiffusion}
over $(t,t+\delta]$, with initialization $\obY_t=\bY_t(t)$.}

\rev{Finally we point out that the construction presented year can be easily modified to 
yield stochastic localization processes that do not fit the time reversal construction of 
Section \ref{sec:FirstDiffusions}. For instance we could observe the process $\bZ_s$
 not on a sub-interval $[\fs(t),S]$ as in Eq.~\eqref{eq:ReverseDiffusion} 
 but on a more general set increasing with $t$.}
 
\subsection{The Eucldean invariant process}
 \label{sec:Euclidean}
 
 \rev{Assume that $\bx$ represents the positions of $m+1$ particles (e.g. atoms in 
 a molecule) $\bx^{(0)},\bx^{(1)},\dots,\bx^{(m)}\in\reals^d$. We are interested in these 
 positions up to an element of the euclidean group $E(d)$. In order to remove the invariance 
 under translations, we can assume without loss of generality that $\bx^{(0)} = \bzero$.
 We thus set $\bx = (\bx^{(1)},\dots,\bx^{(m)})$, $n=md$ and are inderested in sampling
 from $\mu$ which is invariant under simultaneous
 orthogonal transformations $\bx\mapsto \bT_{\bU}\bx:= (\bU\bx^{(1)},\dots,\bU\bx^{(m)})$,
 $\bU\in\cuO(d)$ (the orthogonal group in $\reals^d$).}
 
 \rev{It can be useful to construct the observation process $\bY_t$ in
 such a way that $\mu_t$ remains invariant for all $t$. One possible approach is
 to make the observation process depend only on their Gram matrix $\bG(\bx)\in\reals^{m\times m}$, 
 $G_{ij}(\bx) = \<\bx^{(i)},\bx^{(j)}\>$. The simplest example is of course}
 \begin{align}
 \bY_t = t\bG(\bx) +\bB_t\, ,
 \end{align}
 {where $(\bB_t)$ is a standard Brownian motion valued in $m\times m$ matrices.}
 
 \rev{The resulting scheme generate a random Gram matrix $\bG=\bG(\bx)$,
 for $\bx\sim \mu$. This can be postprocessed to obtain $\bx$ via 
 decomposition $\bG=\bX^{\sT}\bX$ for $\bX\in\reals^{d\times m}$, and then reading out 
 the columns of $\bX$ as particles' positions. }
  
\subsection{All of the above}
\label{sec:All}

One useful property of the present approach is that it provides 
a natural way to combine two observation processes in a new one.
Namely, given observation processes $(\bY^{(1)}_t)_{t\in I_1}$,  $(\bY^{(2)}_t)_{t\in I_2}$,
with $I_1,I_2\subseteq\reals_{\ge 0}$,
we can combine them by defining
\begin{align}
\bY_{t}=\big(\bY^{(1)}_s,\bY^{(2)}_s:\; s\le t \big)\, .
\end{align}
This clearly satisfies the conditions of Definition \ref{def:GeneralDef}.

\rev{As an example let $\bY^{(1)}_t$ be the anisotropic Gaussian process of 
section \ref{sec:Isotropic}, and $\bY^{(2)}_t$ the erasure process of Section 
\ref{sec:Erasure}, which we modify by letting $T_i\sim \Exp(1)$ instead of $T_i\sim \Unif([0,1])$
so that the process $(\bY^{(2)}_t)$ is indexed by $t\in[0,\infty]$. 
The choice $T_i\sim \Exp(1)$  corresponds to the following picture. In a small time interval 
$[t,t+\delta]$, each coordinate $x_i$ is revealed independently with probability $\delta+o(\delta)$.
If it was revealed already before, $Y^{(2)}_t$ does not change.
We thus have}
\begin{align}
\bY_{t}^{(1)} &= t\bx\, + \bW_{t}\, ,\\
Y^{(2)}_{t,i} &= \begin{cases}
x_i & \mbox{ if }t\ge T_i,\\
\star & \mbox{ if }t< T_i.
\end{cases}
\end{align}
\rev{and can represent $\bY_t$ simply as the concatenation of these two vectors
 $\bY_t = (\bY^{(1)}_t,\bY^{(2)}_t)$ (because we can check that the future is independent 
 of the past given these vectors.)}

\rev{The two processes $\bY_{t}^{(1)}$, $\bY_{t}^{(2)}$ are conditionally independent given 
$\bx$, but are not independent unconditionally. As a consequence, the transition probabilities
are ot in product form. Intuitively, the fact that we observe both
$\bY_{t}^{(1)}$ and $\bY_{t}^{(2)}$ changes the conditional distribution of $\bx$
given $\bY_{t}^{(1)}$ and $\bY_{t}^{(2)}$, and hence the conditional distribution of
$\bY_{t+\delta}=(\bY_{t+\delta}^{(1)},\bY_{t+\delta}^{(2)})$.}

\rev{It is simple to derive the new transition probabilities for the present example.
The component $\bY^{(1)}_t$ satisfies $\bY^{(1)}_0=\bzero$ and}
\begin{align}
\de\bY^{(1)}_t &= \bm(\bY^{(1)}_t,\bY^{(2)}_t;t)\, \de t+\de\bB_t\, ,\\
\bm(\by^{(1)},\by^{(2)};t) &= \E\big[\bx\big|t\bx+\sqrt{t}\bG = \by^{(1)}; \bx|_{S(\by^{(2)})}=\by^{(2)}|_{S(\by^{(2)})}\big]\, ,
\end{align}
\rev{where $S(\by^{(2)}):=\{i\in [n]: \, y^{(2)}_i\neq \star\}$ is the set of coordinates
that are not `erased' in $\by^{(2)}$, and $\bv|_S$ denote the restriction of vector $\bv$ to
coordinates in $S$. In words, the SDE \eqref{eq:Isotropic} is modified in that the drift 
$\bm(\bY_t^{(1)},\bY_t^{(2)};t)$ takes into account the observation of $\bY^{(2)}_t$.}

\rev{The evolution of $\bY^{(2)}_t$ is modified similarly. It is initialized as
$\bY^{(2)}_0=(\star,\dots,\star)$ and, for a small time interval $\delta$, and for each $i\in[n]$}
\begin{align}
Y_{i,t}^{(2)} = \star\;\;&\Rightarrow\;\;
Y_{i,t+\delta}^{(2)} = \begin{cases}
\star & \mbox{ with probability $1-\delta-o(\delta)$,}\\
x^*_i  & \mbox{ with probability $\delta+o(\delta)$,}
\end{cases}\\
Y_{i,t}^{(2)} \neq \star\;\;&\Rightarrow\;\;
Y_{i,t+\delta}^{(2)} = Y_{i,t}^{(2)} \, ,
\end{align}
where 
\begin{align}
x^*_i\sim \prob\big(x_i\in\,\cdot\,|\bY^{(1)}_t,\bY^{(2)}_t\big)\, .
\end{align}

\rev{Unsurprisingly, generating the process $\bY_t = (\bY^{(1)}_t,\bY^{(2)}_t)$
require computing expectations and marginal probabilities with respect to
$\mu_t(\,\cdot\, )= \prob\big(\bx\in\,\cdot\,|\bY^{(1)}_t,\bY^{(2)}_t\big)$.
For the sake of clarity, we write the form of this conditional for the case in which 
$x_i$ takes value in a finite alphabet $\cX\subseteq \reals$ (using $\mu$, $\mu_t$ for the corresponding 
probability mass
 functions):}
\begin{align}
\mu_t(\bx) = \frac{1}{Z(\bY_t)}\, \mu(\bx)\, \exp\Big\{
-\frac{1}{2t}\big\|t\bx-\bY^{(1)}_t\big\|\Big\}\, \prod_{i:Y_{i,t}^{(2)}\neq\star}
\bfone_{x_i=Y_{i,t}^{(2)}}\, .
\end{align}
\rev{with $Z(\bY_t)$ fixed by the normalization condition $\sum_{\bx\in\cX^n}\mu_t(\bx)$.
More generally, if the $\bY^{(1)}_t$,  $\bY^{(2)}_t$ have conditional densities 
$\sP_{1,t}(\,\cdot\,|\bx)$,
$\sP_{2,t}(\,\cdot\,|\bx)$ with respect to some fixed reference measure $\nu$, then}
\begin{align}
\mu_t(\de \bx) = \frac{1}{Z(\bY_t)}\mu(\de\bx)  
\cdot \sP_{1,t}(\bY^{(1)}_t|\bx)\cdot \sP_{2,t}(\bY^{(2)}_t|\bx)\, .
\end{align}

%
%
\section{The role of the sampling scheme: Generating from a mixture}
\label{sec:Multimodal}

\rev{This section presents a simple illustration of the
importance the choice of the specific stochastic localization scheme.
We consider the problem of sampling from a mixture of two Gaussians in high-dimension,
which we regard as an instance of high-dimensional distribution with a low-dimensional
latent variable.}
\rev{We show that the isotropic Gaussian process of Section
\ref{sec:Isotropic} runs into difficulties because the denoiser $\bm(\,\cdot\,,t)$
must be estimated accurately in a low-probability region. In contrast, a stochastic Localization
process that is aware of the low-dimensional latent structure succeeds much more easily. }

More concretely, we consider a mixture of two well-separated Gaussians in $n$ dimensions with centers
$\ba_1$, $\ba_2\in\reals^n$, and weights $p_1=p$, $p_2=1-p$.
For simplicity, we will assume the two Gaussians to have common (known) covariance
that therefore we can assume to be equal to identity, and that the overall mean 
$p\ba_1+(1-p)\ba_2$ is known. Therefore the mean can be removed from the data
and we are left with the simple model
\begin{align}
\mu = p\cdot \normal((1-p)\ba;\id_n) +  (1-p)\cdot \normal(-p\ba;\id_n) \, .
\end{align}
where $\ba := \ba_1-\ba_2$. We will further assume that $p\in (0,1)$ is independent of $n$,
 and that the radius of each of these Gaussians (which is of order $\sqrt{n}$) is of the same order
as the norm $\|\ba\|_2$. (These assumptions are mainly introduced for convenience of presentation.)

In Figure \ref{fig:Mixture_Isotropic}, we display attempts to sample from $\mu$ using 
isotropic diffusions,
i.e. the process of Eq.~\eqref{eq:SLESDE}. We use $\ba = \bfone$, $p=0.7$, $n=128$.
Each row is obtained using a different model for the posterior expectation
$\bm(\by;t)$, and reports the histogram 
of $\<\bX_t,\ba\>/\|\ba\|_2^2$ obtained by $1000$ independent runs of the generation process.
Here $\bX_t=\bm(\bY_t;t)$ is the sample generated at time $t$.
These empirical results are compared with the correct distribution 
$\<\bX,\ba\>/\|\ba\|_2^2$  under $\bX\sim\mu$.

\begin{algorithm}
\caption{Forward function; 2-Layer fully connected denoiser (for Gaussian mixture)}\label{alg:FirstFC}
\DontPrintSemicolon
\SetKwFunction{Forward}{Forward}
\SetKwProg{Fn}{Function}{:}{}
\Fn{\Forward{$\bx\in\reals^{n}, \alpha\in [0,\pi/2]$}}{
$\bphi \gets (\cos(\alpha\cdot i), \sin(\alpha\cdot i); i\le 20 )$\;
$\bs \gets \Lin_0(\bphi)$\;
$\bx_1\gets   \ReLU\circ \Lin_1(\bx)$\;
$\bx_2 \gets {\sf Flatten}(\bs\otimes\bx_1)$\;
$\bx_{\out} \gets \cos(\alpha)  \bx+\Lin_2(\bx_2)\in \reals^{n}$\;
\Return{$\bx_{\out}$}\;
}
\end{algorithm}
The four models used  to generate data in Figure \ref{fig:Mixture_Isotropic}
have the same architecture, namely a two-layer fully connected ReLU network
with $m$ hidden nodes, three $L\times 20$ linear layers encode time dependence,
and a skip connection. 
Pseudocode for this architecture is given as Algorithm \ref{alg:FirstFC}, whereby:
\begin{itemize}
\item We encode $t$ in terms of the angle variable $\alpha = \arctan(1/\sqrt{t})$.
\item $\Lin_i$ is a fully connected linear map,
with $\Lin_0:\reals^{40}\to\reals^{L}$, $\Lin_1:\reals^{n}\to\reals^{m}$,
 $\Lin_2:\reals^{mn}\to\reals^{n}$.
\item $\otimes$ denotes tensor (outer) product. 
\end{itemize}

We trained on $N$ samples from $\mu$.
Parameters
where chosen as follows (from top to bottom):
$(i)$~$N=5,000$, $500$ epochs, $L=3$, $m=256$;
$(ii)$~$N=20,000$, $500$ epochs, $L=3$, $m=256$;
$(iii)$~$N=20,000$, $2000$ epochs, $L=3$, $m=256$;
$(iv)$~$N=20,000$, $2000$ epochs, $L=6$, $m=512$.
 (We refer to Appendix \ref{sec:OmittedSamplingGaussians} for further details.)
 It is quite clear that the distribution generated is significantly different from the target
 one along this one-dimensional projection.  

\begin{figure}[t]
\begin{center}
\includegraphics[width=0.95\linewidth]{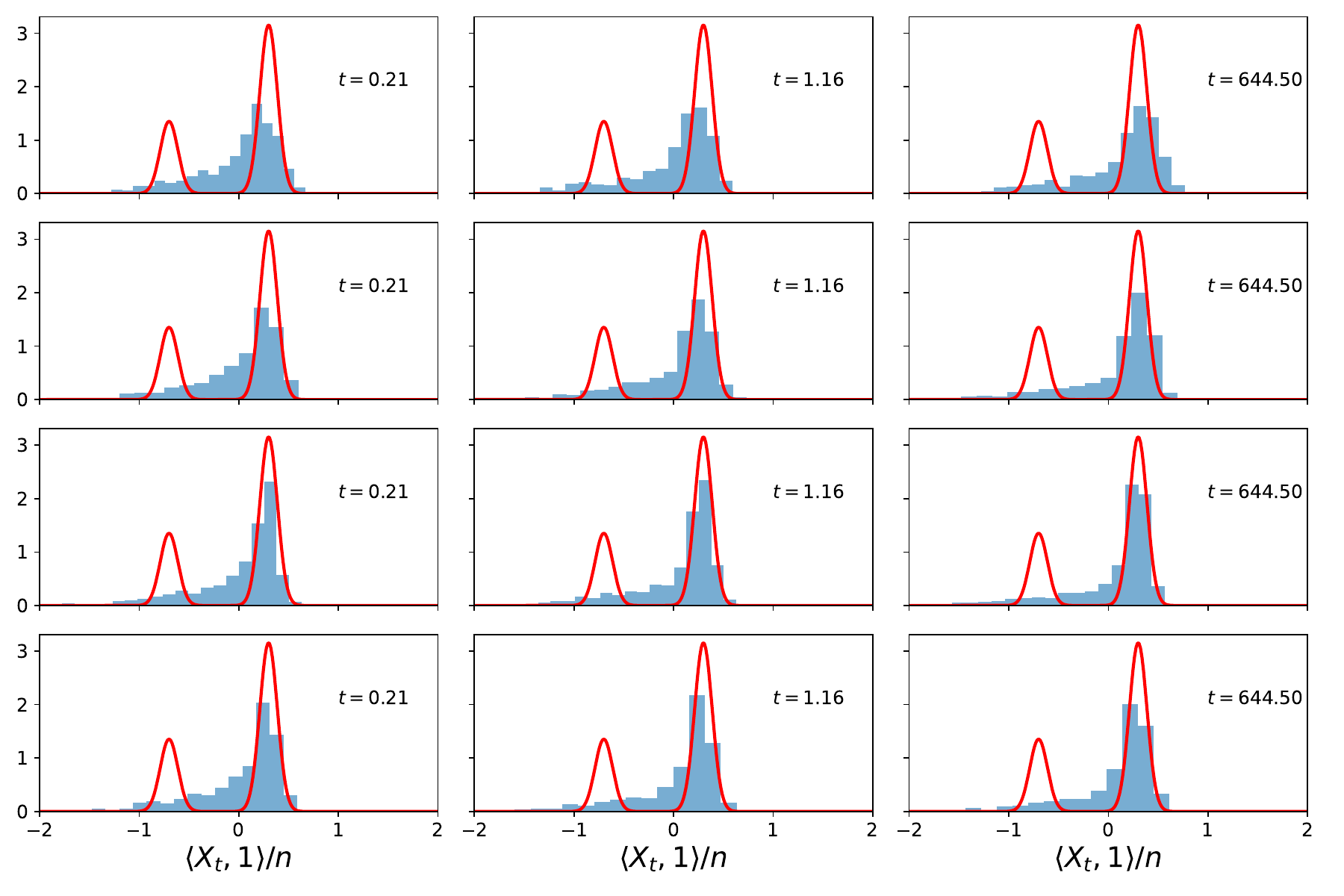}
		\caption{Generating from a mixture of two Gaussians in $n=128$ dimensions using isotropic diffusions
		We compare the empirical distribution of the projection along the direction of the means difference,
		with the correct distribution.
		Each row corresponds to a different model for the posterior mean, and each column to a different
		time in the generation process.}
		\label{fig:Mixture_Isotropic}
\end{center}
\end{figure}

\begin{figure}[t]
\begin{center}
\includegraphics[width=0.95\linewidth]{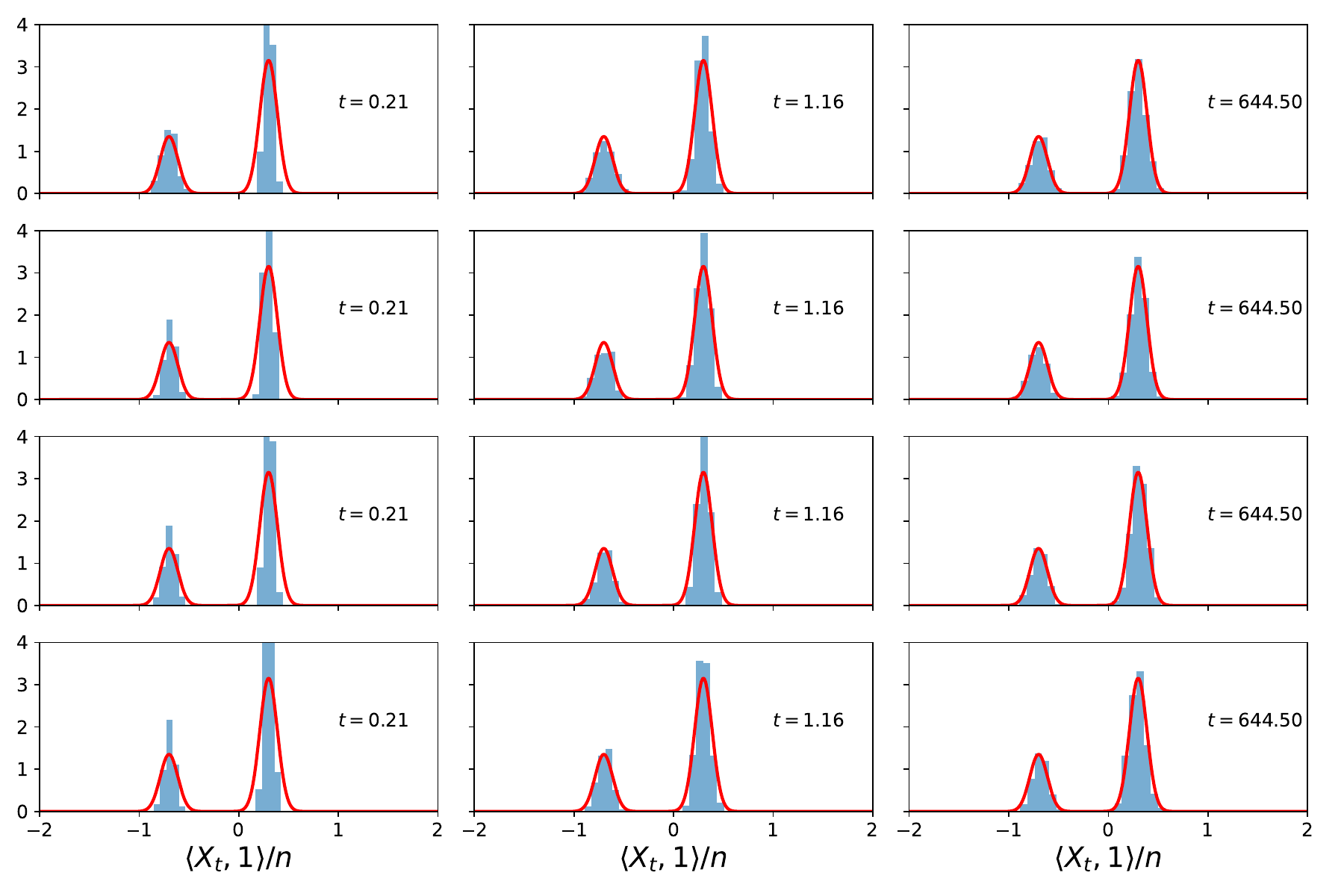}
		\caption{Generating from a mixture of two Gaussians in $n=128$ dimensions.
		The setting and network architecture are the same as in Fig.~\ref{fig:Mixture_Isotropic},
		although the generating process is different. Alongside the original data perturbed by Gaussian noise,
		we reveal $\<\bv,\bx\>$ for a fixed vector $\bv$.}
		\label{fig:Mixture_Obs}
\end{center}
\end{figure}
In Figure \ref{fig:Mixture_Obs} we repeat the experiment keeping the same data generation process
and same network architecture, but introducing a small change in the generation process.
 Given data $\bx_1,\bx_2,\dots,\bx_N$, we compute the principal eigenvector
 $\bv$ of the empirical covariance $\hbSigma:=\sum_{i=1}^n\bx_i\bx_i^{\sT}/n$.
 and the empirical fraction $\hq$ of samples that have a positive projection onto this vector.
Namely, $\hq:=\#\{i\le N:\<\bx_i,\bv\>\ge 0\}/N$. Further, two distinct denoisers
$\bm_+(\by;t)$, $\bm_-(\by;t)$ are learnt respectively 
from the samples $\bx_i$ such that $\<\bx_i,\bv_1\>\ge 0$ and from those such that 
$\<\bx_i,\bv_1\>< 0$. The estimated probability $\hq$ is stored along side the neural networks
for $\bm_+$, $\bm_-$. 

The generation process then proceeds as follows:
\begin{enumerate}
\item Sample $S\in\{+1,-1\}$ with probabilities $\prob(S=+1)=\hq = 1-\prob(S=-1)$.
\item Generate $\bY_t$, $t\ge 0$ running the isotropic diffusion process 
with denoiser $\bm_S(\,\cdot\,;t)$.
\item Return $\bX_T = \bm_S(\bY_T;T)$ for some large $T$. 
\end{enumerate}
It is straightforward to see that this is a combination 
(in the technical sense of Section \ref{sec:All}) of the isotropic process
of Section \ref{sec:Isotropic} and the half-space process of Section \ref{sec:Rectangle}. 
A related idea was developed in \cite{montanari2023posterior}

Figure \ref{fig:Mixture_Obs} demonstrates that the modified process produces a distribution that
matches better the target along direction $\ba$. 
While it is likely that similar results could have been obtained without
changing the sampling scheme, using a more complex architecture to approximate $\bm(\by;t)$,
the new sampling scheme simplifies this task and offers a convenient alternative.

What is the origin of the difficulty in sampling via isotropic diffusions?
A simple calculation shows that the posterior expectation takes the form
\begin{align}
\bm(\by;t) &= \frac{\by}{1+t}+\ba\, \varphi\Big(\frac{\<\ba,\by\>}{\|\ba\|^2};t\Big)\, ,\\
\varphi(s;t) &:= \frac{p(1-p)}{1+t} 
\frac{e^{\big(\frac{(1-p)s}{1+t}-\frac{t(1-p)^2}{2(1+t)}\big)\|\ba\|^2}
-e^{-\big(\frac{ps}{1+t}+\frac{t p^2}{2(1+t)}\big)\|\ba\|^2} }
{pe^{\big(\frac{(1-p)s}{1+t}-\frac{t(1-p)^2}{2(1+t)}\big)\|\ba\|^2 }
+(1-p)e^{-\big(\frac{ps}{1+t}+\frac{t p^2}{2(1+t)}\big)\|\ba\|^2}}\, .
\end{align}
It is easy to see that such a function can be accurately approximated
by a ReLU network with one hidden layer.

However, when $\|a\|= \Theta(\sqrt{n})$ is large, we have the following  behavior.
For any constant $\Delta$
\begin{align}
\bm(\by;t) &= 
\begin{dcases}
\frac{\by+(1-p)\ba}{1+t} +O(1/n)&
 \mbox{ if } \frac{\<\ba,\by\>}{\|\ba\|^2}\ge(1-2p)t +\Delta\, ,\\
\frac{\by-p\ba}{1+t} +O(1/n)&
 \mbox{ if } \frac{\<\ba,\by\>}{\|\ba\|^2}\le (1-2p)t -\Delta\, .
\end{dcases}
\end{align}
Further, the two behaviors are matched on a window of size $\Theta(1/\|\ba\|^2)= \Theta(1/n)$
around  $\<\ba,\by\>/ \|\ba\|^2=(1-2p)t$, which corresponds to the midpoint
between the two cluster centers, scaled by $t$. The derivative of $\varphi$ with respect to its first
argument in this window is positive and of order $n$: as a consequence, the evolution along 
the direction $\ba$ is highly sensitive to correctly estimating $\varphi$.

%
%
\section{The role of architecture: Images with long range correlations}
\label{sec:ToyNum}

In this section, we illustrate
the interplay between sampling process and the network
architecture by considering a simple numerical example. 

We generate synthetic RGB images $\bx_i\in \reals^{3\times w\times h}$, $i\in \{1,\dots, n\}$
according to a simple distribution that is specified in Appendix \ref{sec:OmittedSamplingImages}
with $w=h=512$. This distribution results in images that are
either mostly blue (with probability 1/2) or  mostly red (with probability 1/2),
with some smooth variations. Samples generated according to this distribution are shown in Figure 
\ref{fig:Samples}.

\begin{figure}
\begin{center}
\includegraphics[width=0.5\linewidth]{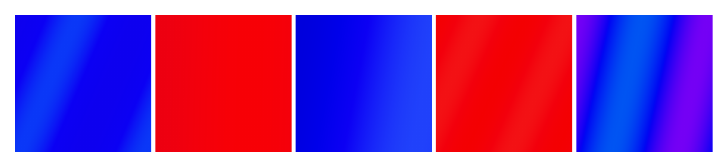}
		\caption{Sample images from the synthetic distribution used in 
		for experiments in Section \ref{sec:ToyNum}. See Appendix \ref{sec:ToyDistribution}
		for a full definition.}
		\label{fig:Samples}
\end{center}
\end{figure}

We try to learn a generative diffusion model for these images 
using two slightly different methods: $(1)$~An isotropic diffusion 
as defined in Section~\ref{sec:Isotropic}; 
$(2)$~A linear observation process
as defined in Section~\ref{sec:Linear}. In both cases, we use a simple 2-layer
convolutional network as denoiser. 

\begin{figure}
\begin{center}
\includegraphics[width=0.75\linewidth]{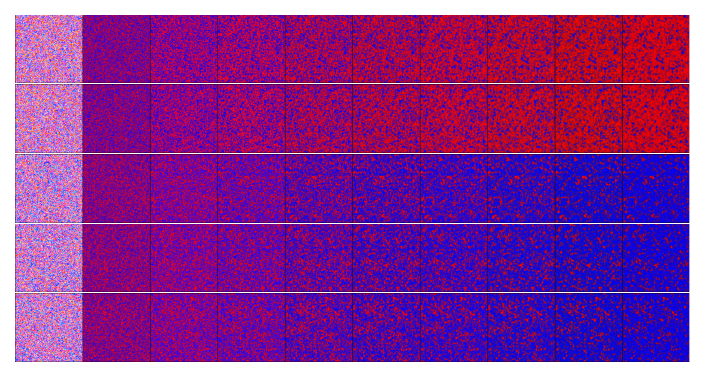}
\includegraphics[width=0.75\linewidth]{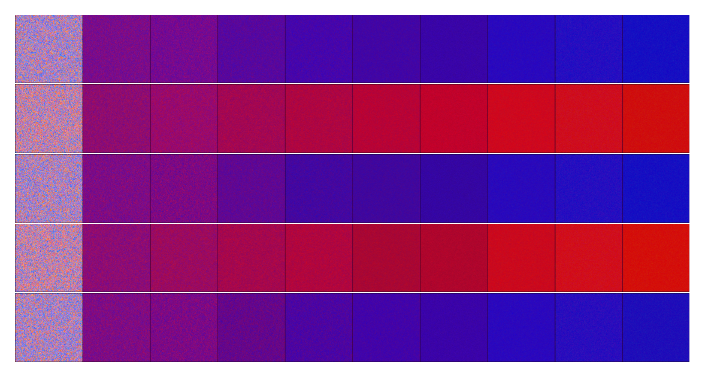}
\end{center}
		\caption{Learning the distribution of Figure \ref{fig:Samples} 
		and sampling via stochastic localization.
		Upper block: Standard isotropic diffusion. Lower block: A linear observation process.
		Each row corresponds to an independent realization of the generating process,
		with time progressing from left to right.}
		\label{fig:Generating}
\end{figure}
Before providing further details, we point to Figure \ref{fig:Generating},
which presents sample trajectories from the two generating 
processes\footnote{In these figures, we randomly flipped Red$\leftrightarrow$Blue in the
whole image with probability $1/2$, to avoid the problems discovered in the previous
section and focus on the effect of convolutional architecture.}. 
Despite the common denoiser architecture, the two generating processes behave very differently. 
Indeed, the isotropic process is locally correlated but misses global correlations. The linear
observation process instead captures these correlations while keeping the same short-ranged 
denoiser 
architecture.

\subsection{Isotropic diffusion}

In our first approach, we train a simple two-layer convolutional neural network to denoise these images,
i.e. we attempt to minimize the empirical risk 
\begin{align}
\hR_n(\btheta) &= \frac{1}{n}\sum_{i=1}^n\E\big\{\big\|\bx_i-
\hbm_{\btheta}(t\bx_i+\sqrt{t}\bG;t)\big\|^2\big\}\, .\label{eq:RiskDenoise}
\end{align}
Here expectation is taken with respect to 
$t = \tan(\alpha)^{-2}$, $\alpha\sim\Unif([0,\pi/2])$, and
 $\bG\sim\normal(\bfzero,\id_{3\times w\times h})$.

\begin{algorithm}
\caption{Forward function; 2-Layer CNN denoiser (standard)}\label{alg:FirstCNN}
\DontPrintSemicolon
\SetKwFunction{Forward}{Forward}
\SetKwProg{Fn}{Function}{:}{}
\Fn{\Forward{$\bx\in\reals^{3\times w\times h}, \alpha\in [0,\pi/2]$}}{
$\bphi \gets (\cos(\alpha\cdot i), \sin(\alpha\cdot i); i\le 10 )$\;
\For{$i \gets 0$ \bf{to} $4$}{
$\bs_i \gets {\sf Lin}_i(\bphi)$
}
 $\bx_1 \gets \bs_0+ \bs_1 \odot {\sf Relu}({\sf Conv}_1(\bx))\in \reals^{k\times w\times h}$\;
$\bx_{\out} \gets \cos(\alpha)  \cdot \bs_2 \odot \bx+\bs_3 \odot \bx + \bs_4 \odot {\sf Conv}_2(\bx_1)\in \reals^{3\times w\times h}$\;
\Return{$\bx_{\out}$}\;
}
\end{algorithm}

We use a simple convolutional network with one hidden layer.
Pseudocode for this network is given as Algorithm \ref{alg:FirstCNN}.
A few details: 
\begin{itemize}
\item We encode $t$ in terms of the angle variable $\alpha = \arctan(1/\sqrt{t})$.
\item $\Lin_i$ is a fully connected linear map.
\item $\Conv_i$ is a two-dimensional convolution with window size $5$.
  $\Conv_1$ has $k=12$ channels and $\Conv_2$ has $3$ channels.
\item $\odot$ denotes entrywise product with the PyTorch broadcasting convention. 
\end{itemize}

Since the convolutional layers have window size $5\times 5$, 
each output pixel $i$ in $\hbm_{\btheta}(\by;t)$ is a function of a $9\times 9$
patch around $i$ in the input $\by$. This appears to
result in the short range correlations in the images of Fig.~\ref{fig:Generating}.

\subsection{Linear observation process}
\label{sec:LobsExperiments}

\begin{algorithm}
\caption{Forward function; 2-Layer CNN denoiser (linear obs.)}\label{alg:SecondCNN}
\DontPrintSemicolon
\SetKwFunction{Forward}{Forward}
\SetKwProg{Fn}{Function}{:}{}
\Fn{\Forward{$\bz\in\reals^{6\times w\times h}, \alpha\in [0,\pi/2]$}}{
$\bphi \gets (\cos(\alpha\cdot i), \sin(\alpha\cdot i); i\le 10 )$\;
\For{$i \gets 0$ \bf{to} $4$}{
$\bs_i \gets {\sf Lin}_i(\bphi)$
}
 $\bx_1 \gets \bs_0+ \bs_1 \odot {\sf Relu}({\sf Conv}_1(x))\in \reals^{k\times w\times h}$\;
 $\bx_2 \gets \cos(\alpha)  \cdot \bs_2 \odot \bx+\bs_3 \odot \bx + \bs_4 \odot {\sf Conv}_2(\bx_1)\in \reals^{k\times w\times h}$\;
$\bz_{\out}\gets \bL \bx_2$\;
\Return{$\bz_{\out}$}\;
}
\end{algorithm}

In our second approach, we use the linear observation process
of Section~\ref{sec:Linear}, with the simple linear operator
$\bL:\reals^{3\times w\times h}\to \reals^{3\times w\times h}\times \reals^3$
defined as follows. Writing $\bx = (x_{ijl}:\, i\le 3,j\le w,l\le h)$ for the entries 
of $\bx$
\begin{align}
\bL\bx = \left(\begin{matrix}
\bx \\
b\bx^{\sav}
\end{matrix}\right)\, , \;\;\;\;\;\;\;\;\;\;  x_i^{\sav}:=\frac{1}{wh}\sum_{j\le w,l\le h}x_{ijl}\, 
,\;\; i\in\{1,2,3\}\, .
\end{align}
In  words, the operator $\bL$ appends to $\bx$ the averages of the 
values in each of the three channels, scaled by a factor $b$. In our simulations we use
$b=2$. Roughly speaking, the factor $b$ implies that information about the average 
$\bx^{\sav}$ is revealed a factor $b$ faster than about other coordinates in $\bx$.

The corresponding generative process \eqref{eq:ADiffusion} can be implemented with 
minimal modifications with respect to the previous section. 
Namely, we encode $\bL\bx$ as a tensor $\bz\in \reals^{6\times w\times h}$
with $6$ channels, whereby channels $4,5,6$ are constant in space and contain the averages 
$bx_1^{\sav}$,$bx_2^{\sav}$, $bx_3^{\sav}$.
In the generative process, we add the same noise to all entries in each of these channels.

The denoiser in this case is detailed in Algorithm \ref{alg:SecondCNN}
and presents minimal modifications with respect to the one in the previous section:
\begin{itemize}
\item The input has $6$ channels instead of $3$.
\item Correspondingly, the middle layer has $15$ channels instead of $12$.
\item At the output we enforce that channels $4,5,6$ contain the mean of previous ones by 
applying the operator $\bL$.
\end{itemize}

As shown in Figure \ref{fig:Generating}, these minimal modification produce a significantly 
different behavior. The images generated show stronger long range correlations.
%
%
\section{The role of architecture:  Shift-invariant Gaussians}
\label{sec:ToyAnalytical}

In order to gain more understanding of the example in the previous section, 
we consider the toy problem 
of samplig from a centered Gaussian $\mu=\normal(\bzero,\bSigma)$.
We will focus on the case in which $\bSigma$ is
a symmetric circulant matrix. Namely there exists $c:\integers\to\reals$, such that
$\Sigma_{i,j} = c(i-j)$ and $c(k+n)=c(k)$, $c(k)=c(-k)$ for all $k$. 
Equivalently, $\mu$ is a centered Gaussian process invariant under shifts.

As a running example, we will consider the case $\bSigma=\id+\alpha\bfone\bfone^{\sT}$.
In other words, for $g_0\sim\normal(0,1)$ independent of $\bg\sim\normal(\bfzero,\id_n)$,
we have 
\begin{align}
\bx = \sqrt{\alpha} g_0\bfone+ \bg\, .\label{eq:GlobalCorr}
\end{align}
The condition number is $\kappa(\bSigma)= (1+n\alpha)$. 

It is worth reminding that a standard sampling procedure would be to
generate $\bg\sim\normal(\bzero,\id)$, and then let $\bx=\bSigma^{-1/2}\bg$. 
We will make no attempt to beat this simple method, and will instead study the
behavior of other approaches on this example.

\subsection{Sampling via isotropic Gaussian diffusions}

The posterior expectation
defined in Eq. \eqref{eq:Denoiser} is a linear function of $\by$. A simple calculation yields
\begin{align}
\bm(\by;t) =\bA_t \by\, ,\;\;\;\;\bA_t := (\id+t\bSigma)^{-1}\bSigma\, ,
\end{align}
and therefore $\bY_t$ satisfies the simple SDE
\begin{align}
\de \bY_t =\bA_t \bY_t\de t+\de\bG_t\, .\label{eq:LinearY}
\end{align}
Of course,  $\bY_t$ is normal with mean zero and covariance $t^2\bSigma+t\id$.
Letting $\bX_t:=\bm(\bY_t;t)$, we have $\bX_t\sim\normal(\bzero, \bSigma_t)$
where
\begin{align}
\bSigma_t = (1+t\bSigma)^{-1}t\bSigma^2\, .
\end{align}
whence an accurate approximation\footnote{For instance,
we can compute $W_2(\mu,\mu_t)^2=\sum_{i=1}^n\lambda_i\,f(\lambda_i t)$,
where $f(x) = (1-x/\sqrt{1+x^2})^2$.} of $\mu$ is achieved for $t\gtrsim 1/\lambda_{\min}(\bSigma)$.
Note that accurate discretization requires stepsize $\delta \lambda_{\max}(\bSigma)\lesssim 1$,
and therefore the total number of iterations will scale as the condition number $\kappa(\bSigma):=
\lambda_{\max}(\bSigma)/\lambda_{\min}(\bSigma)$.

In general, the denoiser $\bm(\by;t)$ will be replaced by an approximation.
How does architecture of the denoise  impact the generated distribution?
Since the distribution $\mu$ is Gaussian and shift-invariant, it is natural
to use a convolutional linear denoiser. However, we will constrain the convolution 
 window size to be
$2r+1\ll n$. Namely we use a matrix in $\cL(r,n)$, where
\begin{align}
\cL(r,n) :=\Big\{\bM\in\reals^{n\times n}: \; 
M_{i,j} = \ell_M(i-j)\, , \;\ell_M(k) =\ell_M(-k) =  \ell_M(n+k) \forall k 
\, , \; \ell_M(k)=0\forall |k|>r\Big\}\, .
\end{align}
We learn such a denoiser by minimizing the mean square error:
\begin{align}
\bA_t^{(r)} :=\argmin
\Big\{
\E\Big[\big\|\bx-\bA(t\bx+\sqrt{t}\bg)\big\|^2_2\Big]\,:\;\;\;
\bA\in\cL(r,n)
\Big\}\, .
\end{align}
A simple calculation reveals 
$(\bA_t^{(r)})_{i,j} =\ell_t(|i-j|)$
where $(\ell_t(u))_{-r\le u\le r}$ solves 
\begin{align}
\ell_t(u) +t \sum_{v=-r}^rc(u-v)\ell_t(v) = c(u)\,,\label{eq:OptConvolution}
\end{align}
with $\ell_t(-u) = \ell_t(u)$. 
Given a solution of this equation we can determine the distribution of $\bY_t$
(by integrating Eq.~\eqref{eq:LinearY} whereby $\bA_t$ is replaced by $\bA^{(r)}_t$)
and hence the distribution of $\bX_t=\bL^{(r)}_t\bY_t$. 

It follows from the symmetries of the problem that
 $\bX_t\sim \normal(\bzero,\bSigma^{\sgen}_t)$ where $\bSigma^{\sgen}_t$ is
a symmetric circulant matrix.
We limit ourselves to giving the results of this calculation when the correlation
structure is given by \eqref{eq:GlobalCorr}. We get 
$\lim_{t\to\infty}\bSigma^{\sgen}_t = \bSigma^{\sgen}$ where 
$\Sigma^{\sgen}_{ij}  =c^{\sgen}(i-j)$ and
 \begin{align}
 c^{\sgen}(\ell) &= \frac{1}{n}\sum_{q\in B_n} \hat{c}^{\sgen}(q)\, e^{iq\ell}\, ,
 \;\;\;\;\;
 B_n:=\Big\{q=\frac{2\pi k}{n}: -(n/2)+1\le k\le (n/2)\Big\}\, ,\\
 \hat{c}^{\sgen}(q) &= F(\nu(q),c_0)\, ,\\ c_0&:=\frac{1}{1+(2r+1)\alpha}\, ,\;\;\;\; \nu(q) := 
 \frac{\sin (q(r+1/2))}{(2r+1)\sin(q/2)}\, ,
 \end{align}
 where $F:\reals\times\reals\to\reals$ is defined in Appendix \ref{sec:OmittedSampling}.
The only fact that we will use is that $x\mapsto F(x;c_0)$ is differentiable at $x=1$,
with $F(1;c_0)=1/c_0$, $F'(1;c_0)>0$.
 
 We claim that the the generated distribution 
 $\mu^{\sgen} = \normal(\bfzero,\bSigma^{\sgen})$ is very far from the target one
  $\mu = \normal(\bfzero,\bSigma)$.  The fundamental reason for this is that --as 
  in the numerical example of the last section-- the measure $\mu$ has long range correlations
  (indeed $\E_{\mu}(x_ix_j)=\alpha>0$ for any $i\neq j$) while the finite width 
  convolutional denoiser cannot produce such long-range correlations. 
  
  These remarks are formalized by the statement below.
  \begin{proposition}
For any fixed $r\in\naturals$, $\alpha>0$, let  $\mu_n = \normal(\bfzero,\bSigma_n)$
be the Gaussian measure, with covariance $\bSigma_n=\id_n+\alpha\bfone_n\bfone_n^{\sT}$,
and denote by $\mu^{\sgen}_{n,r}$ be the generative distribution produced by the diffusion
sampler with convolutional denoiser of window size $2r+1$.
 
 Then we have, for all $(2r+1)\le n/8$ and $n\alpha\ge 4$
 \begin{align}
 W_2(\mu_n,\mu^{\sgen}_{n,r}) & \ge \frac{1}{2}\sqrt{n\alpha}\, ,
 \label{eq:LemmaFirstClaim}\\
 \lim_{n\to\infty}\big\|\mu_n-\mu^{\sgen}_{n,r}\big\|_{\sTV}& = 1\, .
 \label{eq:LemmaSecondClaim}
 \end{align}
  \end{proposition}
  \begin{proof}
For any coupling $\gamma$ of $\mu_n$,$\mu^{\sgen}_{n,r}$, letting 
$(\bx,\bx^{\sgen})\sim\gamma$, we have
\begin{align*}
\E\big\{\big\|\bx-\bx^{\sgen}\big\|_2^2\big\}^{1/2}&\ge \frac{1}{\sqrt{n}}
\E\big\{\<\bx-\bx^{\sgen},\bfone\>^2\big\}^{1/2}\\
&\ge \frac{1}{\sqrt{n}}
\E\big\{\<\bx,\bfone\>^2\big\}^{1/2} -\frac{1}{\sqrt{n}}
\E\big\{\<\bx^{\sgen},\bfone\>^2\big\}^{1/2}\\
& = \sqrt{\frac{1}{n}\<\bfone,\bSigma_{n}\bfone\>}- \sqrt{\frac{1}{n}\<\bfone,\bSigma^{\sgen}_{n,r}\bfone\>}\, .
\end{align*}
Here the first inequality follows from Cauchy-Schwarz and the second is
 triangular inequality.
 On the other hand, using the above formulas
 \begin{align}
 \<\bfone,\bSigma_{n}\bfone\>  &= \alpha n^2+n\, ,\\
 \<\bfone,\bSigma^{\sgen}_{n,r}\bfone\> & = \hat{c}^{\sgen}(0) n = F(1,c_0)n =
 (1+(2r+1)\alpha)n\, .
 \end{align}
 Substituting in the above, and using the definition of Wasserstein distance,
 we have
 \begin{align}
 W_2(\mu_n,\mu^{\sgen}_{n,r}) & \ge \sqrt{n\alpha+1}-\sqrt{(2r+1)\alpha+1}\, ,
 \end{align}
 and the claim \eqref{eq:LemmaFirstClaim} follows by a simple calculation.

To prove Eq.~\eqref{eq:LemmaSecondClaim}, define the random variables
$Z :=  \<\bx,\bfone\>/n$, $Z^{\sgen} :=  \<\bx^{\sgen},\bfone\>/n$.
By the above calculation we have $Z\sim \normal(0,\alpha+n^{-1})$,
$Z^{\sgen}\sim \normal(0,(1+(2r+1)\alpha)n)$ and therefore 
\begin{align}
\big\|\mu_n-\mu^{\sgen}_{n,r}\big\|_{\sTV} \ge 
\big\|P_Z-P_{Z^{\sgen}} \big\|_{\sTV}\to 1\, .
\end{align}
  \end{proof}

The notion that the measure $\mu^{\sgen}$ has only short range correlations
can be easily made more precise. A simple calculation shows that correlations only extend to
distances of order $r$, see Appendix \ref{sec:OmittedSampling}.

\subsection{Sampling via the linear observation process}

We consider the same linear observation process as in our numerical experiments of 
Section \ref{sec:LobsExperiments} (with obvious adaptations).
Namely, the observation process is defined by 
\begin{align}
\bY_{t} &= \left(\begin{matrix}Y_{0,t}\\
\bY_{\star,t}\end{matrix}\right)\, , \;\;\;\; 
Y_{0,t}= \frac{bt}{n}\<\bx,\bfone\>+ B_{0,t}\, ,\;\;
\bY_{\star,t}= t\, \bx+\bB_{\star,t}\, ,
\end{align}
where $\{(B_{0,t},\bB_{\star,t})\}_{t\ge 0}$ is an $(n+1)$-dimensional Brownian motion. 
This corresponds to the general construction of Section~\ref{sec:Linear}, whereby the matrix
$\bL$ is given by
\begin{align}
\bL = \left[\begin{matrix}
b/n & b/n & b/n & \cdots & b/n\\
1 & 0 & 0 & \cdots & 0\\
0 & 1 & 0 & \cdots & 0\\
\cdot & \cdot & \cdot & \cdots & \cdot\\
\cdot & \cdot & \cdot & \cdots & \cdot\\
0 & 0 & 0 & \cdots & 1\\
\end{matrix}
\right]\, .
\end{align}

We follow the general approach of Section~\ref{sec:Linear} for sampling, namely
\begin{align}
\de \bY_t =\bm_{\bL}(\bY_t;t)\de t+\de\bG_t\, ,
\end{align}
where now $\bm_{\bL}(\,\cdot\,;t):\reals^{n+1}\to\reals^n$, and $(\bG_t)_{g\ge 0}$
is an  $(n+1)$-dimensional Brownian motion, and $\bm_{\bL}(\by;t)=\E[\bL\bx|\bY_t=\by]$. 
In fact we know that the correct $\bm$ is linear
(since the distribution of $\bx$ is Gaussian) and shift invariant (since the distribution
of $\bx$ shift-invariant). It is understood that $\ell$-th shift acts on $\reals^{n+1}
=\reals\times\reals^n$ via
\begin{align}
\bS^{\ell}\left(\begin{matrix}
z_0\\
z_1\\
z_2\\
\vdots\\
z_n
\end{matrix}\right) = 
\left(\begin{matrix}
z_0\\
z_{1+\ell}\\
z_{2+\ell}\\
\vdots\\
z_{\ell}
\end{matrix}\right)\, .
\end{align}
In other words, we can always consider (writing $\by=(y_0,\by_{\star})$)
\begin{align}
\bm_{\bL}(y_0,\by_{\star};t) & =: 
\left(\begin{matrix}
m_{0}(y_0,\by_{\star};t)\\
\bm_{\star}(y_0,\by_{\star};t)
\end{matrix}\right) = 
\left(\begin{matrix}
d_t & c_t \bfone^{\sT}\\
a_t\bfone & \bA_t 
\end{matrix}\right)\cdot
\left(\begin{matrix}y_0\\
\by_{\star}
\end{matrix}\right)\, ,
\end{align}
where $\bA_t$ is a circulant matrix and $a_t,c_t,d_t$ are scalars. 

With the objective of understanding the experiments of Section \ref{sec:LobsExperiments},
we attempt to approximate the optimal $\bA_t$ using finite-window covolutions:
\begin{align}
\big(a_t^{(r)},c_t^{(r)},d_t^{(r)},\bA_t^{(r)}\big) :=\argmin
\left \{
\E\left\{\left\|\bL\bx- \left(\begin{matrix}
d & c \bfone^{\sT}\\
a\bfone & \bA
\end{matrix}\right)\cdot
\left(\begin{matrix}Y_{0,t}\\
\bY_{\star,t}
\end{matrix}\right)\right\|^2_2\right\}\,:\;\;\;
\bA\in\cL(r,n)
\right\}\, .
\end{align}
where $Y_{0,t}, \bY_{\star,t}$ are distributed as specified above.

To simplify calculations, rather than explicitly solving the above quadratic optimization
problem, we will guess a good feasible solution and check whether it gives the desired 
probability approximation.
Our guess will be
\begin{align}
\hm_{0}(y_0,\by_{\star};t) & = \frac{\alpha b^2 y_0}{1+\alpha b^2 t}\, ,\label{eq:Guess1}\\
\hbm_{\star} (y_0,\by_{\star};t) & =  \frac{1}{1+t}
\Big(\by_{\star}- \frac{\alpha b y_0}{1+\alpha b^2 t}
\bfone\Big)+\frac{\alpha b y_0}{1+\alpha b^2 t}
\bfone\, .\label{eq:Guess2}
\end{align}
The rationale for this choice is as follows. Recall that $\bx = \sqrt{\alpha} g_0\bfone+\bg$
and therefore letting $\bz=\bL\bx$, we have $z_0= \sqrt{\alpha} b g_0+(G_1/\sqrt{n})$  with 
$(g_0,G_1)$ independent standard normals. Therefore, the proposed 
$\hm_{0}(y_0,\by_{\star};t)$ is the Bayes optimal estimator for $z_0$ given $Y_{0,t}$
up to terms $O(1/n)$. As for the $\hbm_{\star} (y_0,\by_{\star};t)$, notice that (for $i\ge 1$)
$z_i-(z_0/b)\sim\normal(0,1)$. Hence, if $z_0$ was known, the optimal estimator for 
$z_i$ would be $(z_0/b)+(y_i-z_0/b)/(1+t)$. The estimator \eqref{eq:Guess2}
replaces $z_0$ by its estimate given by Eq.~\eqref{eq:Guess1}.
\begin{proposition}\label{prop:LobsIsGood}
 Let  $\mu_n = \normal(\bfzero,\bSigma_n)$
be the Gaussian measure, with covariance $\bSigma_n=\id_n+\alpha\bfone_n\bfone_n^{\sT}$,
and denote by $\hmu^{\sgen,t}_{n}$ the distribution generated by the 
(continuous time) linear observation process (with 
estimators of Eq.~\eqref{eq:Guess1}, \eqref{eq:Guess2}) at time $t$. 

Then:
\begin{align}
\lim_{t\to\infty}W_{2}(\mu_{n},\hmu^{\sgen,t}_{n}) = \lim_{t\to\infty}\|\mu_{n}-\hmu^{\sgen,t}_{n}
\|_{\sTV}=0\, .
\end{align}
\end{proposition}

\section{Approximating the transition probabilities}
\label{sec:Oracle}

Sampling schemes based on diffusions or stochastic localization 
reduce the problem of sampling from the target distribution $\mu$ 
to the one of approximating the transition probabilities
\begin{align}
\rP_{t,t+\delta}(\by, A)& =\prob\big(\bY_{t+\delta}\in A\big|\bY_t=\by\big)\, .
\end{align}
In this section we discuss this problems having in mind the two scenarios described
in Section \ref{sec:FirstSampling}: 
$(1)$~$\mu$ is known analytically (but approximating
the transition probability is nevertheless nontrivial);
$(2)$~$\mu$ is only known via samples $\bx^{(1)},\dots, \bx^{(N)}\sim_{i.i.d.}\mu$,
and the transition probabilities need to be learnt from data.

If $\hrP_{t,t+\delta}$ are approximate transition probabilities, then it is 
useful to measure their quality via the Kullback-Leibler 
(KL) divergence between the two processes:
\begin{align}
D(\rP\|\hrP):= -\E_{\rP} \log \frac{\de\hrP}{\de\rP}(\bY_0^T)\, .\label{eq:GeneralKL}
\end{align}
Here, with an abuse of notation, we denote by $\rP$, $\hrP$ the probability 
measures over the paths $\bY_0^T$ induced by the above transition probabilities.
In order to simplify formulas, we will assume $Y_0$ to have the same distribution under 
measures $\rP$ and $\hrP$. 
It is often reasonable to assume that $P$ and $\hP$ have densities with respect to
the a common measure $\rP_0$ (not necessarily normalized), 
and parametrize the density of the latter as 
$f(\by_0^T;\btheta)$. We thus obtain
\begin{align}
D(\rP\|\hrP):= -\E_{\rP} \log f(\by_0^T;\btheta)-D(P\|\rP_0)\, .\label{eq:KL-Parametrized}
\end{align}

\subsection{Approximating transition probabilities for explicitly known $\mu$}

\rev{If $\mu$ is explicitly known (e.g. is of the form \eqref{eq:ExponentialForm}),
then we try to construct algorithmic approximations of the transition probabilities.
For most of the examples in Section \ref{sec:Examples}, this amounts to computing the mean or
marginal distribution of a single coordinate of $\bx$ under the conditional measure
$\mu_t(\,\cdot\,):=\prob(\bx\in\,\cdot\,|\bY_t)$. 
As specific localization scheme  thus provides  a reduction from the problem of sampling to
the problem of computing the mean or marginals  of the original distribution, conditioned
to $\bY_t$. Two possible approaches for this task are $(i)$~Variational methods;
$(ii)$~Method based on correlation decay.}

\rev{For illustration, let us consider sampling from
\eqref{eq:ExponentialForm} using the isotropic process of section \ref{sec:Isotropic},
in which case we need to estimate the mean $\bm(\by,t)$.} 

\rev{\emph{Variational methods}
starts from the observation that, by the Gibbs variational principle}
\begin{align}
\bm(\by,t) & = \argmin\Big\{-\<\by,\bm\>+\frac{t}{2}\<\bfone,\bs\>+\cuF(\bm,\bs)\Big\}\, ,\\
\cuF(\bm,\bs)&:= \min\Big\{\int H(\bx)\nu(\de\bx)-\Ent(\nu): \int \!\bx\, \nu(\de\bx)=\bm,
\int \!\bx^{\odot 2}\, \nu(\de\bx)=\bs\Big\}\, ,
\end{align}
\rev{where $\bx^{\odot 2}\in\reals^n$ denotes the entrywise square of $\bx$,
and  $\Ent(\nu)$ is the entropy of $\nu$ with respect to Lebesgue\footnote{Namely $\Ent(\nu)=\infty$ 
unless $\nu$ has a density $f$ with respect to Lebesgue,
in which case $\Ent(\nu) = \int f\log(1/f)\de\bx$.}.
Variational methods construct tractable approximations of $\cuF(\bm,\bs)$.
We refer to  \cite{el2022sampling,el2025sampling,montanari2023posterior,huang2024sampling} for applications of this approach to high-dimensional 
statistics and statistical physics.}

\rev{\emph{Correlation decay methods} apply to cases in which $H(\bx)$ decomposes according to a 
graph structure, e.g.  $H(\bx)=\sum_{(i,j)\in E} U_{ij}(x_i,x_j)$, where $E$ 
is the edge set of a finite graph $G=(V,E)$. The measure $\mu_t$ then decomposes according the
same graph}
\begin{align}
\mu_t(\de\bx) = \frac{1}{Z(\bY_t,t)} \,\exp\Big\{-\sum_{(i,j)\in E}U_{ij}(x_i,x_j)+
\sum_{i\in V}\Big(Y_{t,i}x_i-\frac{t}{2}x_i^2\Big)\Big\}\, \de\bx\, .
\end{align}
\rev{Roughly speaking, correlation decay methods attempt to compute $m_i(\bT_t,t)=\int x_i \, \mu_t(\de\bx)$
replacing $\mu_t$ by the measure $\mu^{(i)}_t$ obtained by including only variables
in a neighborhood $V(i)\subseteq V$ of $i$  \cite{weitz2006counting,dembo2010gibbs,gamarnik2012correlation}.}

\subsection{Estimating transition probabilities from samples}

\rev{Given data $(\bx^{(i)})_{i\le N}\sim_{iid}\mu$, we generate realizations 
$(\bY^{(i)}_{t})_{i\le N,t\ge 0}$
of the observation process. We can then estimate a parametric model 
$\of$ by minimizing the empirical version of Eq.~\eqref{eq:KL-Parametrized}
(negative likelihood)}
\begin{align}
\hR_n(\btheta) := -\frac{1}{N}\sum_{i=1}^N \log\of\big((\by^{(i)})_0^T;\btheta\big)\, .
\label{eq:EmpiricalRisk}
\end{align}
\rev{Consider for simplicity the case of discrete time $t\in I=\{0,\delta,2\delta,\dots ,T\}$.
Then we can set 
\begin{align}
\of\big(\by_0^T;\btheta\big) = \prod_{t\in I\setminus\{T\}}f_t\big(\by_t,\by_{t+\delta};\btheta\big)
\end{align}
where $f_t(\by_t,\by_{t+\delta};\btheta)$ models the conditional probability density 
of $\bY_{t+\delta}$ given $\bY_t=\by_t$. The empirical risk then reads}
\begin{align}
\hR_n(\btheta) := -\frac{1}{N}\sum_{i=1}^N\sum_{t\in  I\setminus \{T\}} 
\log f_t\big(\by^{(i)}_t,\by^{(i)}_{t+1};\btheta\big)\, .
\end{align}
\rev{In the special case of the isotropic Gaussian process of Section
\ref{sec:Isotropic}, it is natural to approximate the posterior mean $\bm(\by;t)$
by $\hbm(\by,t;\btheta)$ (so that $\hrP$ is the measure of a diffusion with drift $\hbm$).
The cost  \eqref{eq:EmpiricalRisk} is then essentially equivalent to the standard score  
matching objective (see Appendix \ref{sec:Loss})}
\begin{align}
\hR_n(\btheta) := -\frac{1}{N}\sum_{i=1}^N\int_0^{\infty} 
\big\|\hbm(\by_t^{(i)},t;\btheta)-\bx^{(i)}\big\|\, \de t\, .
\end{align}
\rev{In Appendix \ref{sec:Loss} we also derive explicit forms of the 
loss function for all the examples of Section \ref{sec:Examples}.}

\subsection{On the choice of the stochastic localization scheme}

\rev{We briefly mention a few aspects to consider when choosing a suitable
stochastic localization scheme, for sampling a distribution $\mu$ of interest.}

\paragraph{Preserve symmetries.} \rev{If the distribution $\mu$ is invariant under a transformation
group $\fG$, it can be useful to choose an observation process $\bY_t$
which is also invariant\footnote{If $g\in \fG$ acts on $\reals^n$, as 
$\bx\mapsto \sT_{g}\bx$, then $\mu$ is invariant if $\mu(\sT_gA)=\mu(A)$
for any set $A$, and the observation process $\bY_t$ is invariant if $(\bY_t)_{t\in I}|_{\bx}\ed
(\bY_t)_{t\in I}|_{\sT_g\bx}$.} under $\fG$. If this is the case, then
$\mu_t$ remains invariant at all $t$. An example is provided in Section
\ref{sec:Euclidean}.}

\rev{Constructing the observation process $\bY_t$ so as to
respect the problem symmetries is useful because it allows to focus on
sampling (and estimating)  the `non-trivial' degrees of freedom.}

\paragraph{Avoid hard regions.} \rev{Consider --to be definite-- the isotropic Gaussian
process of Section \ref{sec:Isotropic}. A substantial literature 
(see, e.g., \cite{brennan2018reducibility,hopkins2017power,schramm2022computational})
 provides
evidence of the fact that ---in some cases--- the denoiser 
$\bm(\by,t)$ cannot be approximated\footnote{Namely,
no algorithm $\hbm$ achieves $\E\{\|\hbm(\bY_t,t)-\bm(\bY_t,t)\|^2\} =o(1)\cdot
\E\{\|\bm(\bY_t,t)\|^2\}$.} using polynomial time algorithms, for some interval of 
values of $t$. This phenomenon is referred to as an `information-computation gap.'
In practice $\bm(\,\cdot\,,t)$ will be approximated by a neural network or
another polynomial time algorithm. Hence the intractability of $\bm(\by,t)$
leads to a failure of sampling fail, as shown in a simple case in \cite{montanari2025computational}.}

\rev{In general, we want to construct the observation process $\bY_t$
such that the transition probabilities $\sP_{t,t+\delta}(\by|\,\cdot\,)$ can be approximated
in polynomial time.}

\paragraph{Avoid phase transitions.}\rev{Of course, rigorous evidence of hardness
exists only in a small number of idealized models. For more complex probability distributions
$\mu$ (which might only be known via samples $\bx^{(1)},\dots, \bx^{(N)}$) it
is a priori unclear whether the denoiser $\bm(\by,t)$ or 
(more generally) the transition probabilities $\sP_{t,t+\delta}(\by|\,\cdot\,)$  
can be approximated by a polynomial time.}

\rev{Let us focus, to be definite, on the problem of approximating the denoiser 
$\bm(\by,t)$. It is a general observation that, when this is intractable in an interval 
$t\in (t_{1},t_2)$, then the minimum mean square error $\MMSE_n(t) =\E\{\|\bx-\bm(\bY_t,t)\|^2\}$
undergoes a phase transition at $t_1$, and similarly for the polynomially achievable mean square error
(at a different point.) Namely $[\MMSE_n((1-\eps_n)t_1)-\MMSE_n((1+\eps_n)t_1)]/\MMSE_n((1-\eps_n)t_1)$
remains bounded away from zero as $n\to\infty$ (for some $\eps_n\to 0$.)}

\rev{The co-occurrence of information-computation gaps and phase transitions of the type
just describe suggests to avoid stochastic localization schemes for which the algorithmic
mean square error appears to have sharp drops when $t$ crosses a threshold.}

\paragraph{Add latents.} \rev{If $\mu$ is strongly multimodal or, more generally, it is 
only well concentrated after conditioning on the latent, it is advisable to
take this into account in the construction of the observation process. 
Intuitively, sampling the latent from the correct marginal distribution because 
the latent values have significant impact on samples.
A naive sampling scheme can fail at this because (for instance) 
the denoiser $\bm(\by,t)$ is trained to capture the marginal distribution of individual 
variables later than the global latent. An example of such a failure
was given in Section \ref{sec:Multimodal}.}

\rev{In order to address this difficulty it is advisable to add observations (as part
of the observation process) that explicitly depend on the latents.}

\section*{Acknowledgements}

I would like to thank Marc Laugharn for a stimulating collaboration that motivated the present 
report.

\bibliographystyle{amsalpha}
\addcontentsline{toc}{section}{References}

\newcommand{\etalchar}[1]{$^{#1}$}
\providecommand{\bysame}{\leavevmode\hbox to3em{\hrulefill}\thinspace}
\providecommand{\MR}{\relax\ifhmode\unskip\space\fi MR }
\providecommand{\MRhref}[2]{%
  \href{http://www.ams.org/mathscinet-getitem?mr=#1}{#2}
}
\providecommand{\href}[2]{#2}

\newpage 

\appendix

%
%
\section{Loss functions}
\label{sec:Loss}
The purpose of this appendix is to collect known explicit formulas for 
the general KL divergence \eqref{eq:GeneralKL}. In particular, we will cover 
all examples detailed in the previous sections.

\subsection{Gaussian observation process}

Consider the Gaussian process of Section \ref{sec:Anisotropic},
which we stop at time $T$. We use the estimated drift $\hbm(\by;\bOmega)$
to generate
\begin{align}
\de\hbY_t &= \bQ(t) \bm(\hbY_t;\bOmega(t))\, \de t + \bQ(t)^{1/2} \de \bB_t\, .
\end{align}
An immediate application of Girsanov's theorem yileds
\begin{align}
D(P\|\hP) = \frac{1}{2}\int_{0}^T\big\|\bQ(t)^{1/2}(\bm(\bY_t;\bOmega(t))-\hbm(\bY_t;
\bOmega(t)))\big\|_2^2
\de t\, .
\end{align}
This of course includes the standard diffusion of Section \ref{sec:Isotropic}
as a special case.
Also, the linear information process of Section \ref{sec:Linear}
also fits this framework if we reinterpret $\bm(\bY_t;t)$  as $\bm_{\bA}(\bY_t;t)$.

\subsection{Discrete time Markov chains}

Assume $t\in I:= \{0,1,\dots,T\}$. We denote the transition probabilities by 
$\rP_t(\by;A):= \prob(\bY_{t+1}\in A |\bY_t=\by)$ and 
$\hrP_t(\by;A):= \hprob(\bY_{t+1}\in A|\bY_t=\by)$.

If $\hrP_t(\by;\,\cdot\, )$ has a density with respect to $\rP_t(\by;\,\cdot\,)$,
then we have:
\begin{align}
D(P\|\hP) = -\sum_{t=0}^T\E
\log \frac{\de\hrP_t}{\de\rP_t}(\bY_t;\bY_{t+1})\, .
\end{align}

For instance this is the case for the erasure process of Section
\ref{sec:Erasure}. In this case, the transition probabilities are
estimates of the conditional laws
$\mu\big(x_{i(t)}\in \,\cdot\, \big|x_{i(1)},\dots, x_{i(t-1)}\big)$
\begin{align}
D(P\|\hP) = -\sum_{t=0}^T\E
\log \frac{\de\hmu}{\de\mu}\big(X_{i(t)}|X_{i(1)},\dots, X_{i(t-1)}\big)\, .
\end{align}
Similarly, for the information percolation process of Section
\ref{sec:Percolation}, we have
\begin{align}
D(P\|\hP) = -\sum_{t=0}^T\E
\log \frac{\de\hmu_t}{\de\hmu}(x_{t(\ell+1)}-x_{o(\ell+1)} \big|
x_{t(1)}-x_{o(1)},\dots, x_{t(\ell)}-x_{o(\ell)}\big)\, .
\end{align}

\subsection{Continuos time Markov chains}

We will consider the case of discrete state space $\bY_t\in\cY$.
Assume $t\in I =[0,T]$, and transition rated given by
\begin{align}
\prob\big(\bY_{t+\delta} = \by'|\bY_t = \by\big) & = \begin{cases}
1-q_t(\by)\delta+o(\delta) & \mbox{ if $\by'=\by$,}\\
r_t(\by,\by')\delta+o(\delta) & \mbox{ if $\by'\neq \by$,}
\end{cases}\\
q_t(\by)&:=\sum_{\by'\neq \by} r_t(\by,\by')
\end{align}

Then we have,
\begin{align}
D(P\|\hP) =\int_{0}^T \E\big\{ \sum_{\by\neq \bY_t}
\Delta\big(r_t(\bY_t,\by)\| \hr_t(\bY_t,by)\big)\big\}\, \de t
\end{align}
where $\Delta(r\|\hr)$ is KL divergence between two Poisson random variables
of means $r$ and $\hr$, namely 
\begin{align}
\Delta(r\|\hr) =  r\, \log\frac{r}{\hr}
-r+\hr\, .
\end{align}

As a special case, we have the symmetric process of Section  \ref{sec:BSC}
to generate $\bx\in\{+1,-1\}^n$.
Recall that in this case, we need an estimate $\hm_i(t;\by)$ of
the conditional expectation $m_i(t;\by) = \E[x_i|\bx\otimes\bZ_t=\by]$
(where $Z_{t,i}\in\{+1,-1\}$, $\E Z_{t,i}=t$).
We use it to form the probabilities
\begin{align}
\hp_i(\by;t) &= \frac{1+t^2}{2 t(1-t^2)}-\frac{1}{1-t^2}y_i\, \hm_i(t;\by)\, .
\end{align}
Substituting in the above, we have:
\begin{align}
D(P\|\hP) &=\int_{0}^T\sum_{i=1}^n\E\big\{ 
\Delta\big(p_i(\bY_t;t)\| \hp_i(\bY_t;t)\big)\big\}\, \de t\, .
\end{align}
It is immediate to generalize this formula for the symmetric process of Section 
\ref{sec:SC}.

Another special case is given by the Poisson process of Section \ref{sec:Poisson}.
In this case, the KL divergence is 
\begin{align}
D(P\|\hP) &=\int_{0}^T\sum_{i=1}^n\E\big\{ 
\Delta\big(m_i(\bY_t;t)\| \hm_i(\bY_t;t)\big)\big\}\, \de t\, .
\end{align}
%

%
%
\section{Sampling Mixtures: Omitted technical details}
\label{sec:OmittedSamplingGaussians}

We used stochastic gradient descent with batch size $50$ over 
a number of epochs and samples that changes depending on the row of Figures 
\ref{fig:Mixture_Isotropic}, \ref{fig:Mixture_Obs}.
A fixed number of samples is generated, while the noise $\bG$ and signal-to-noise ratio
$t$ are resampled independently at each SGD sample.
At training time we sample $t= \tan(\alpha)^{-2}$ with $\alpha\sim\Unif([0,\pi/2])$.
At generation time we use a Euler discretization with $K$ equi-spaced values
of $\alpha$,  $K\in\{200,400\}$ and check that results are insensitive to the value of $K$.

%
%
\section{Sampling images: Omitted technical details}
\label{sec:OmittedSamplingImages}

\subsection{The distribution over images}
\label{sec:ToyDistribution}

Here we define the distribution over images that was used in the experiments
of Section \ref{sec:ToyNum}. 
It is convenient to recast an image $\bx\in\reals^{3\times w\times h}$
as $\bx = (\bx(i_1,i_2))_{i_1\le w-1,i_2\le h-1}$ where $\bx(i_1,i_2) \in\reals^3$. In other words,
$\bx(i_1,i_2)$ is the RGB encoding of pixel $i_1,i_2$.
We then set
\begin{align}
\bx(i_1,i_2) = \tanh(\bpsi(i_1,i_2))\, ,\;\;\;\;\;
\bpsi(i_1,i_2) = \bpsi_0+ \bpsi_1\cos(q_1i_1+q_2i_2)\, ,
\end{align}
and generate a random image by drawing $(\bpsi_0,\bpsi_1,\bq)$ randomly (here $\bq=(q_1,q_2)$).

More specifically, in our experiments we took these three vectors to be independent with
\begin{align}
\bpsi_0 &=\begin{cases}
(1.95, 0, 0.05) & \mbox{with probability $1/2$,}\\
(0.05, 0, 1.95) & \mbox{with probability $1/2$,}\\
\end{cases}\\
\bpsi_1 & \sim \normal\big(\bzero,(1/16)\id_3\big)\, ,\\
\bq & = \Big(\frac{4\pi}{w}\, U_1,\frac{4\pi}{h}\, U_2\Big)\, ,\;\;\;\; U_1,,U_2\sim \Unif([0,1])\, .
\end{align}

\subsection{Some details of the training}

We used stochastic gradient descent with batch size $4$ over $n=300$ samples
(images generated according to the model in the previous section), for $100$ epochs.
While these samples are kept fixed, the noise vector $\bG$ and signal-to-noise ratio
$t$ (cf. ~Eq.~\eqref{eq:RiskDenoise}) are resampled independently at each SGD sample.

%
%
\section{Shift-invariant Gaussians: Omitted derivations}
\label{sec:OmittedSampling}

\subsection{General formulas for isotropic diffusions}

The optimal convolution with window size $2r+1$
is obtained by solving Eq.~\eqref{eq:OptConvolution}. This results in 
\begin{align}
\ell_t(u) &= \sum_{k=0}^{r}\frac{\hc_k}{1+t(2r+1)\hc_k} \cos\left(\frac{(2k+1)\pi u}{2(r+1)}\right)\, .
\end{align}
where the parameters $\hc_k$ characterize the covariance:
\begin{align}
\hc_k := \frac{1}{2r+1}\sum_{u=-r}^r c(u)\, \cos\left(\frac{(2k+1)\pi u}{2(r+1)}\right)\, .
\end{align}

Considering again the example $\bSigma=\id+\alpha\bfone\bfone^{\sT}$,
we obtain, for $b_{r,\alpha}:=1+(2r+1)\alpha$,
\begin{align}
\ell_t(0) &= \frac{1}{1+t}+ \frac{\alpha}{(1+t)(1+b_{\alpha,r}t)}\, ,\\
\ell_t(j) & = \frac{\alpha}{(1+t)(1+b_{\alpha,r}t)}\;\;\;\; \mbox{for $1\le |j|\le r$.}
\end{align}
Again, at any $t$, we get $\bX_t\sim\normal(\bzero,\bSigma^{\bX}_t)$ where $\bSigma_t^{\bX}$
is a circulant matrix with (by construction), with eigenvalue decomposition
\begin{align}
 \bSigma_t^{\bX} &= \sum_{q\in B_n} \sigma^{\bX}_{t}(q)\bphi_q\bphi_q^*\,,\;\;\;\;\; (\bphi_q)_{\ell} = \frac{1}{\sqrt{n}}\, e^{iq\ell}\\
 B_n&:=\Big\{q=\frac{2\pi k}{n}: -(n/2)+1\le k\le (n/2)\Big\}\, ,\, .
 \end{align}
 As $t\to\infty$, we get $ \sigma^{\bX}_{t}(q)\to \sigma^{\bX}(q)$, where
 \begin{align}
 \sigma^{\bX}(q) &= F(\nu(q);c_0)\, ,\\ c_0&:=\frac{1}{1+(2r+1)\alpha}\, ,\;\;\;\; \nu(q) := 
 \frac{\sin (q(r+1/2))}{(2r+1)\sin(q/2)}\, ,
 \end{align}
 and
 \begin{align}
 F(\nu;c) :=  \int_0^{\infty}\!\frac{1}{(1+s)^{2(1-\nu)}(c+s)^{2\nu}}\, \de s \, .
 \end{align}
 For any $c\in (0,1)$ (which is the case when $c=c_0$ defined above),
 the function $\nu\mapsto F(\nu;c)$ is strictly positive, continuously differentiable
 and convex. Further $F(1;c)= 1/c$, $F(0;c)=1$ and, writing $F'(\nu;c)$ for the 
 derivative of $F$ with respect to $\nu$
 \begin{align}
 F'(1;c) = \int_0^{\infty} \frac{2}{(c+s)^2}\log \left(\frac{1+s}{c+s}\right) \, \de s\, .
 \end{align}
 For $s\ge 0$, $c\in (0,1)$, $1\le \log((1+s)/(c+s))\le \log(1/c)$, whence
 \begin{align}
\frac{2}{c}\le F'(1;c) \le \frac{2}{c}\log(1/c)\, .
\end{align}
 
 For $q\to 0$ we have the Taylor expansion
\begin{align}
\sigma^{\bX}(q) = F(1;c_0)-\frac{1}{6}F'(1;c_0)\, r(r+1) \, q^2+O(q^4)\, .
\end{align}

Define the second moment correlation length for the generated process
\begin{align}
\xi_{2}(r):= \lim_{n\to\infty}\sqrt{\frac{\sum_{1\le i,j\le n}\Sigma^{\bX}_{i,j}d_n(i,j)^2}{\sum_{1\le i,j\le n}\Sigma^{\bX}_{i,j}}}
\end{align}
The above implies
\begin{align}
\xi_{2}(r)= \sqrt{\frac{r(r+1)F'(1;c_0)}{6F(1;c_0)}}\, ,
\end{align}
and therefore 
\begin{align}
\sqrt{\frac{r(r+1)}{3}}\le \xi_{2}(r)\le \sqrt{\frac{r(r+1)\log(1+(2r+1)\alpha)}{3}}\, .
\end{align}

\subsection{Proof of Proposition \ref{prop:LobsIsGood}}

The generated process $Y_{0,t}$ satisfies 
\begin{align}
\de Y_{0,t} = \frac{\alpha b^2 Y_{0,t}}{1+\alpha b^2 t}\, \de t +\de B_{0,t}\, ,
\end{align}
which is easily integrated to yield
\begin{align}
Y_{0,t} = \int_0^t\left(\frac{1+\alpha b^2 t }{1+\alpha b^2 s}\right)\, \de B_{0,t}\, ,
\end{align}
In particular, there exists a standard normal random variable $G_0$ such that
the following limit holds almost surely
\begin{align}
\lim_{n\to\infty}\frac{\alpha b^2 Y_{0,t}}{1+\alpha b^2 t} = \sqrt{\alpha b^2}\, G_0\, .
\label{eq:G0Def}
\end{align}

Next consider the generated process $Y_{i,t}$ for any $i\ge 1$:
\begin{align}
\de Y_{i,t} = \frac{1}{1+t} Y_{i,t}\, \de t
+\frac{\alpha b t}{(1+\alpha b^2 t)(1+t)}Y_{0,t}\, \de t +\de B_{i,t}\, ,
\end{align}
which yields
\begin{align}
 Y_{i,t} = 
\int_0^t\frac{1+t}{(1+s)^2}\cdot \frac{\alpha b s}{1+\alpha b^2 s}\, Y_{0,s}\, \de s +
\int_0^t\frac{1+t}{1+s}\, \de B_{i,s}\, ,\label{eq:ReprYit}
\end{align}
In particular,  since $(Y_{0,t})_{t\ge 0}$ is Gaussian and independent of $(B_{i,t})_{i\ge 1, t\ge 0}$,
for any fixed $t$, $\bY_{\star,t} = (Y_{i,t})_{1\le i\le n}$ is centered Gaussian.
Indeed its covariance takes the form $\bSigma_{\bY,t} = c_1(t)\id+c_2(t)\bfone\bfone_{\sT}$
for suitable constants $c_1(t)$, $c_2(t)$. 

Further letting $G_0$ be defined as per Eq.~\eqref{eq:G0Def},
we have the almost sure limit
\begin{align}
\lim_{t\to \infty}\int_0^t\frac{1}{(1+s)^2}\cdot \frac{\alpha b s}{1+\alpha b^2 s}\, Y_{0,s}\, \de s 
&=\int_0^{\infty}\frac{1}{(1+s)^2}\cdot \sqrt{\alpha}\, G_0\, \de s\\ 
&=\sqrt{\alpha}\, G_0\, .
\end{align}
Define $G_i:= \int_0^{\infty} (1+s)^{-1}\, \de B_{i,s}$. We note that
the $(G_i)_{0\le i\le n}$ is a collection of i.i.d. standard normal random variables
and, using the representation \eqref{eq:ReprYit}, almost surely
\begin{align}
\lim_{t\to\infty} \frac{1}{1+t} Y_{i,t} = \sqrt{\alpha}\, G_0+G_i\, .\label{eq:LimYbyT}
\end{align}
Finally, defining 
\begin{align}
\bX_t = \bm_{\star}(Y_{0,t},\bY_{\star,t};t)=\frac{1}{1+t} \bY_{\star,t}\, 
+\frac{\alpha b t}{(1+\alpha b^2 t)(1+t)}Y_{0,t}\, \bfone\, .\label{eq:Xtdef}
\end{align}
and using again Eqs.~\eqref{eq:G0Def} and \ref{eq:LimYbyT}, we obtain
\begin{align}
\lim_{t\to\infty} \bX_t = \sqrt{\alpha}\, G_0\bfone +\bG_*\, .
\end{align}
We note that the right-hand side has the target distribution $\mu_n$, while the left-hand side
is Gaussian for every $t$ (because $\bY_t$ is Gaussian as pointed out above).
Therefore the claim follows from the following standard fact.
\begin{lemma}
Let $(\bZ_k)_{k\ge 1}$ be a sequence of Gaussian vectors and assume $\bZ_k\to \bZ_{\infty}$
almost surely, where $\bZ_{\infty}$ is a non-degenerate Gaussian vector. 
Then, denoting by $\nu_k$ the law of $\bZ_k$, we have
\begin{align}
\lim_{k\to\infty} W_2(\nu_k,\nu_{\infty}) = 0, ,
\;\;\;\; \lim_{k\to\infty} \|\nu_k-\nu_{\infty}\|_{\sTV} = 0\, .
\end{align}
\end{lemma}

\end{document}